\documentclass[10pt,twocolumn,letterpaper]{article}

\usepackage{iccv}              

\usepackage{amsmath,amsfonts,bm}

\def\1{\bm{1}}

\def\vzero{{\bm{0}}}

\def\va{{\bm{a}}}

\def\ve{{\bm{e}}}

\def\vu{{\bm{u}}}

\def\vx{{\bm{x}}}
\def\vy{{\bm{y}}}
\def\vz{{\bm{z}}}

\def\mA{{\bm{A}}}

\def\mE{{\bm{E}}}

\def\mH{{\bm{H}}}
\def\mI{{\bm{I}}}

\def\mP{{\bm{P}}}
\def\mQ{{\bm{Q}}}
\def\mR{{\bm{R}}}

\def\mU{{\bm{U}}}
\def\mV{{\bm{V}}}

\def\mZ{{\bm{Z}}}

\DeclareMathAlphabet{\mathsfit}{\encodingdefault}{\sfdefault}{m}{sl}
\SetMathAlphabet{\mathsfit}{bold}{\encodingdefault}{\sfdefault}{bx}{n}

\newcommand{\E}{\mathbb{E}}

\newcommand{\R}{\mathbb{R}}

\newcommand{\Var}{\mathrm{Var}}

\DeclareMathOperator*{\argmin}{arg\,min}

\usepackage{url}    
\usepackage{graphicx}
\usepackage[utf8]{inputenc} 
\usepackage[T1]{fontenc}    
\usepackage{url}            
\usepackage{booktabs}       
\usepackage{nicefrac}       
\usepackage{microtype}      
\usepackage{xcolor}         
\usepackage{multirow}
\usepackage{makecell} 
\usepackage{pifont}
\usepackage{wrapfig}
\usepackage{threeparttable}

\usepackage{amsthm}
\usepackage{stmaryrd}
\usepackage{algorithm}
\usepackage[noend]{algorithmic}

\usepackage{thm-restate}

\newtheorem{defination}{Defination}

\usepackage{enumitem}
\setlist[enumerate]{leftmargin=*}
\setlist[enumerate,1]{itemsep=0pt,topsep=0pt,parsep=0pt}

\allowdisplaybreaks[1]

\definecolor{iccvblue}{rgb}{0.21,0.49,0.74}
\usepackage[pagebackref,breaklinks,colorlinks,allcolors=iccvblue]{hyperref}

\def\paperID{0000} 
\def\confName{ICCV}
\def\confYear{2025}

\title{Zeroth-Order Fine-Tuning of LLMs in Random Subspaces}

\author{
    Ziming Yu\textsuperscript{1}~,
    Pan Zhou\textsuperscript{2}~,
    Sike Wang\textsuperscript{1}~,
    Jia Li\textsuperscript{1,4}\thanks{Corresponding Author}~,
    Mi Tian\textsuperscript{3}~,
    Hua Huang\textsuperscript{1,4}~ \\
    \textsuperscript{1~}Beijing Normal University \quad    
    \textsuperscript{2~}Singapore Management University \quad
    \textsuperscript{3~}TAL Education Group \\
    \textsuperscript{4~} Engineering Research Center of Intelligent Technology and Educational Application (MOE) \\
    {\parbox[t]{0.8\linewidth}{\centering\tt\small
    \{zimingyu, sikewang\}@mail.bnu.edu.cn,
    \{jiali, huahuang\}@bnu.edu.cn,
    panzhou@smu.edu.sg, tianmi@tal.com}}
}

\begin{document}
\maketitle

\begin{abstract}
Fine-tuning Large Language Models (LLMs) has proven effective for a variety of downstream tasks. However, as LLMs grow in size, the memory demands for backpropagation become increasingly prohibitive. Zeroth-order (ZO) optimization methods offer a memory-efficient alternative by using forward passes to estimate gradients, but the variance of gradient estimates typically scales linearly with the model's parameter dimension—a significant issue for LLMs. In this paper, we propose the random Subspace Zeroth-order (SubZero) optimization to address the challenges posed by LLMs' high dimensionality. We introduce a low-rank perturbation tailored for LLMs that significantly reduces memory consumption while improving performance.  Additionally, we prove  that our gradient estimation closely approximates the backpropagation gradient, exhibits lower variance than traditional ZO methods, and ensures convergence when combined with SGD. Experimental results show that SubZero enhances fine-tuning performance and achieves faster convergence compared to standard ZO approaches like MeZO across various language modeling tasks. Code is available at \url{https://github.com/zimingyy/SubZero}.
\end{abstract}
\section{Introduction} 

Large Language Models (LLMs), such as the GPT and LLaMA series~\citep{zhang2022opt, touvron2023llama}, have recently demonstrated impressive capabilities in natural language processing tasks and beyond~\citep{solaiman2019release,achiam2023gpt}. These models utilize deep learning, particularly the transformer architecture~\citep{vaswani2017attention}, to learn complex patterns  in language data. However, LLMs can struggle with specialized tasks that require domain-specific knowledge~\citep{tag-llm2024}. Fine-tuning presents an effective solution by slightly adjusting pre-trained LLMs with domain data, enabling them to adapt to specific tasks more effectively.

For fine-tuning, first-order (FO) optimizers, such as SGD~\citep{amari1993backpropagation} or Adam~\citep{kingma2014adam}, are commonly used to achieve promising performance on domain datasets.  However,  as LLMs grow in size, FO optimizers demand increasingly memory consumption due to the gradient computations required by backpropagation (BP)~\citep{zhao2024galore}. Additionally, they are unable to directly handle non-differentiable objectives. To enhance memory efficiency, MeZO~\citep{malladi2023fine} first introduces the zeroth-order (ZO)  optimizer to LLM fine-tuning without BP. It only requires forward passes and calculates gradient estimates using finite differences of training loss values, enabling it to directly handle non-differentiable objectives. Nevertheless, the variance of ZO gradient estimates scales linearly with the perturbation dimension, which corresponds to the number of model parameters. This can become extremely large in LLMs, leading to significant performance degradation compared to FO optimizers~\citep{gautamvariance, jiang2024zo, liu2024sparse}. 

Reducing the variance of ZO gradient estimates generally results in faster convergence~\cite{yue2023zeroth}. There are two main attempts to addressing the high variance of ZO gradient estimates.  The first approach involves increasing batch size alongside training steps, which reduces gradient noise and variance  in ZO gradient estimates~\citep{gautamvariance, jiang2024zo}. However, this leads to significant runtime and memory costs due to the large batch size in the later training stages. The second approach focuses on perturbing fewer parameters by employing sparse parameter perturbations, such as random and sparse pruning masks~\citep{liu2024sparse} and block-coordinate perturbations~\citep{zhang2024revisiting}, or by reducing the number of trainable parameters through techniques like parameter-efficient fine-tuning (PEFT)~\citep{malladi2023fine, zhang2024revisiting} and tensorized adapters~\citep{yang2024adazeta}. Recent theoretical advancements have proposed using random projections to lessen the dimensionality dependence in ZO optimizers~\citep{nozawa2024zeroth, roberts2023direct, kozak2021stochastic} by applying low-dimensional perturbations in random subspaces. Nonetheless, a major drawback of this approach is the need to store a huge projection matrix that scales with model parameter dimensionality, making it impractical for fine-tuning large LLMs.

\begin{figure*}[t]	
        \vspace{-0.5em}
	\centering	
	\begin{subfigure}{0.23\textwidth}
		\includegraphics[width=\textwidth]{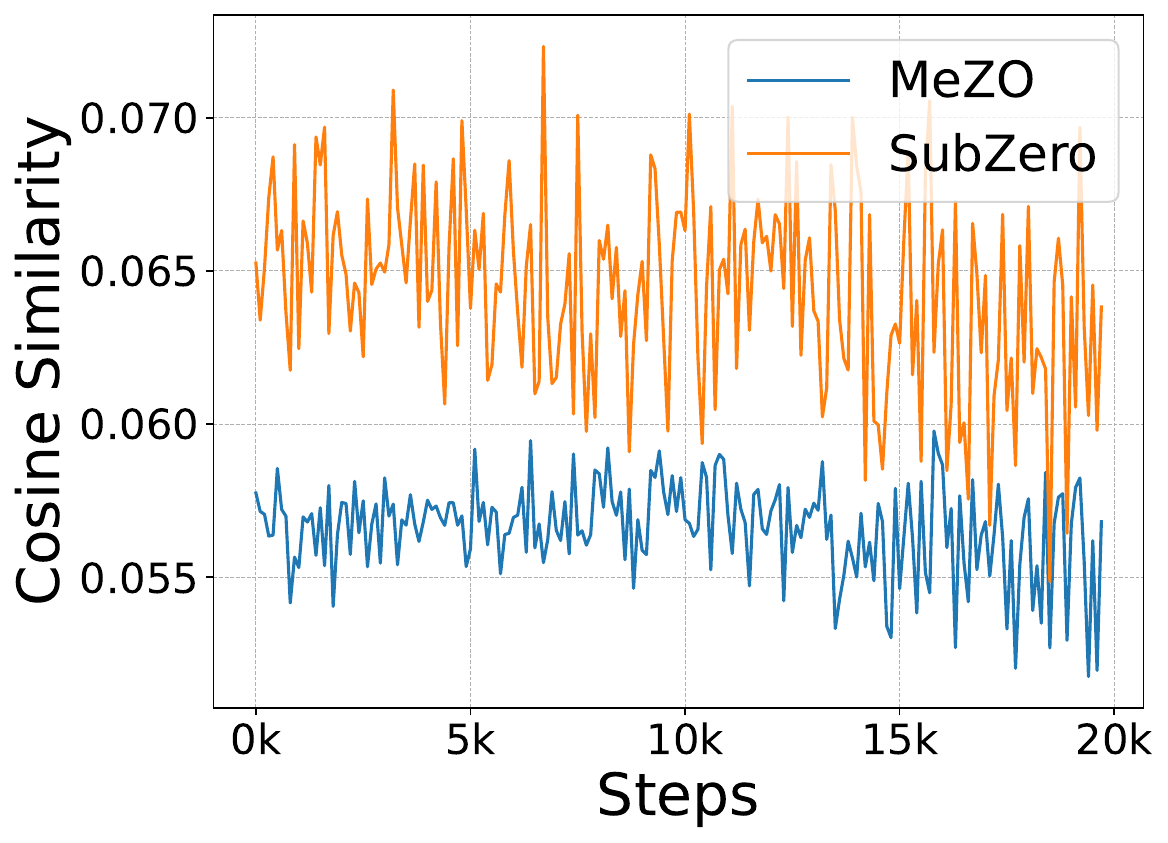}	
        \caption{Cosine Similarity}		
	\end{subfigure}
	\hfill
	\begin{subfigure}{0.22\textwidth}
		\includegraphics[width=\textwidth]{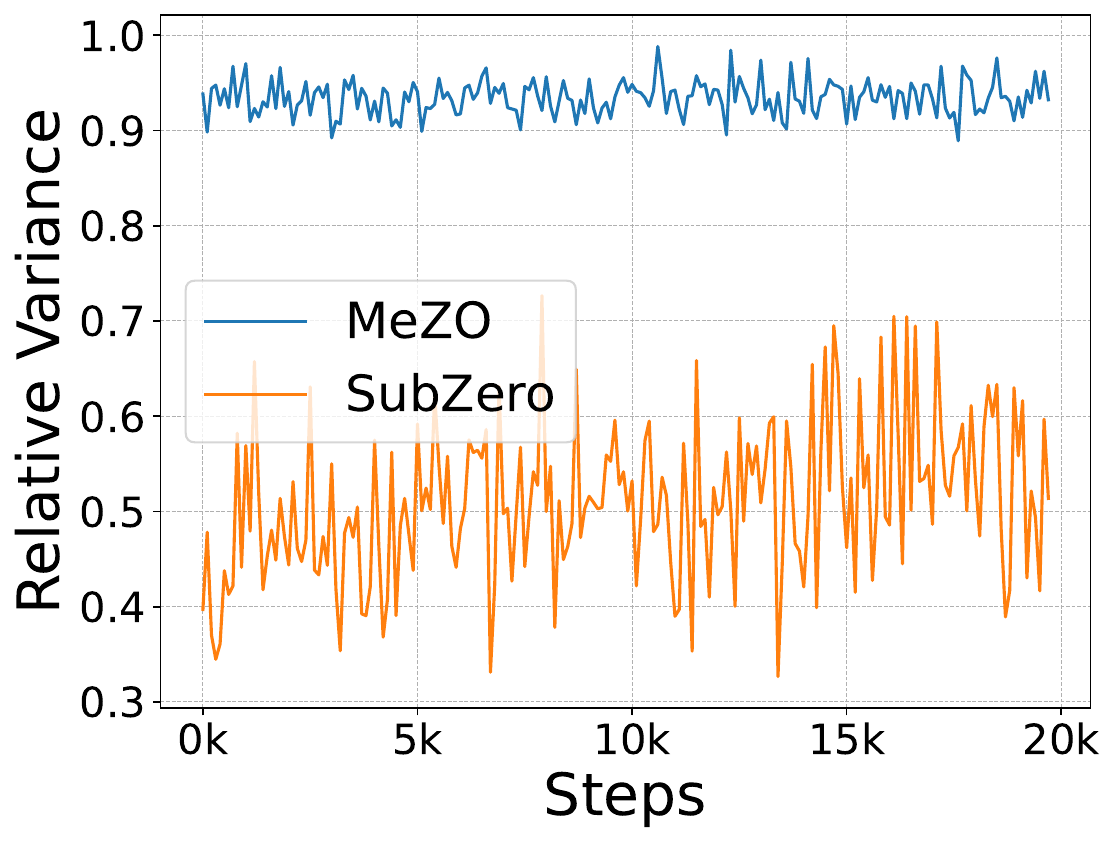}
        \caption{Relative Variance}		
	\end{subfigure}
	\hfill
	\begin{subfigure}{0.22\textwidth}
		\includegraphics[width=\textwidth]{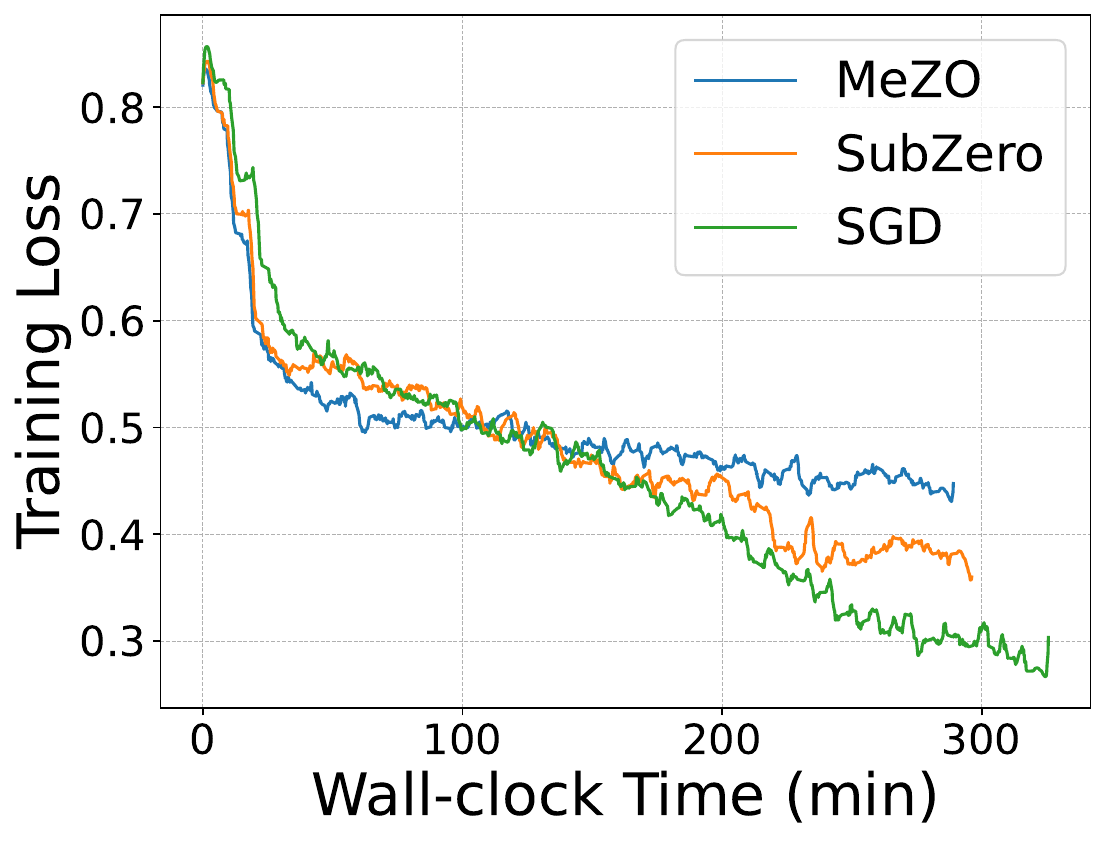}
        \caption{Training Loss}		
	\end{subfigure}
	\hfill
	\begin{subfigure}{0.22\textwidth}
		\includegraphics[width=\textwidth]{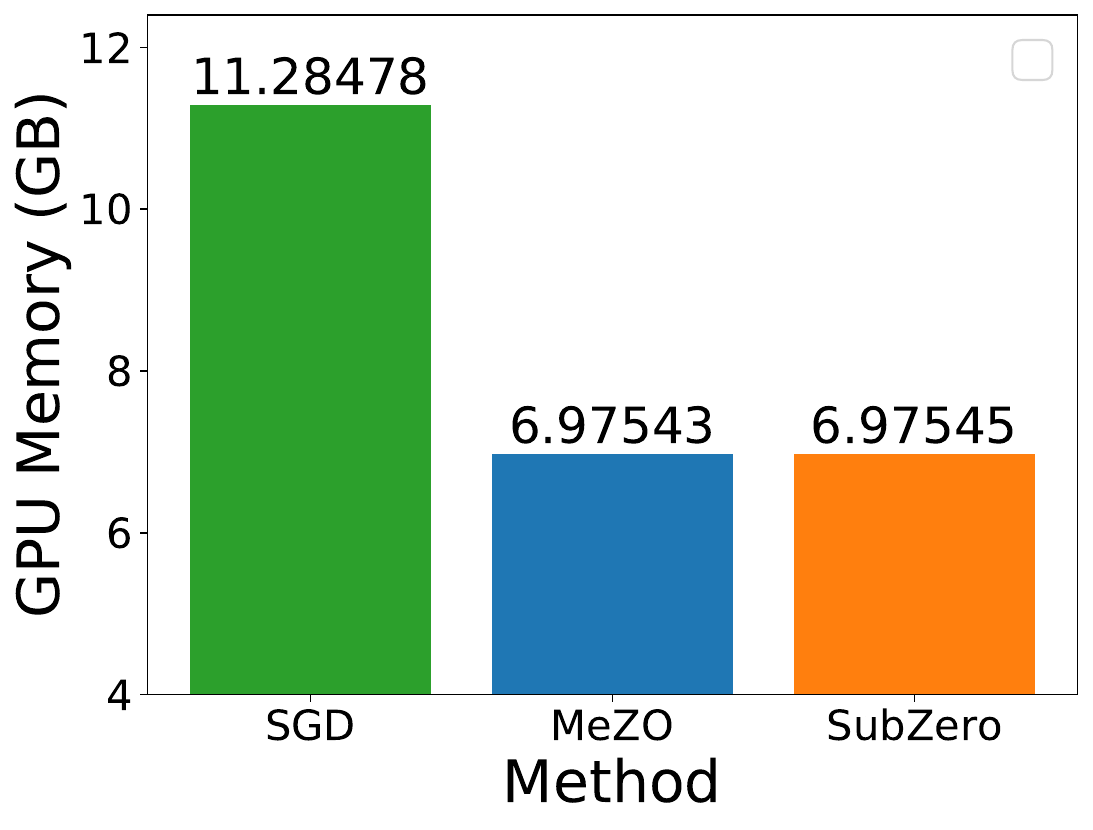}
        \caption{Memory Cost}		
	\end{subfigure}	
	\vspace{-0.8em}
	\caption{  
		Visualization of cosine similarity $\mathbb{E}\left[{\texttt{cosine}(\boldsymbol{g}, \hat{\boldsymbol{g}})}\right]$, relative variance $\Var \bigl[\left \|\hat{\boldsymbol{g}}\right\|\bigr] / \left \| \boldsymbol{g} \right \|^2$, training loss, and peak total GPU memory cost with OPT-1.3B on SST-2 in the prompt tuning scheme. All three methods utilize a batch size of 16 and run for 20K steps. Here, $\hat{\boldsymbol{g}}$ represents the gradient estimated by MeZO or our SubZero, and $\boldsymbol{g}$ denotes the expected gradient  $\mathbb{E}[\hat{\boldsymbol{g}}]$. Theorem~\ref{thm:oracles-mean} (b) ensures that SubZero maintains a small distance between $\boldsymbol{g}$ and the BP gradient in a subspace. (a) and (b) demonstrate that SubZero's estimated gradient $\hat{\boldsymbol{g}}$ has lower angle error and variance than MeZO. (c) and (d) indicate that SubZero enhances convergence speed with minimal extra memory usage. }
	\label{fig:motivation}
	\vspace{-1em}
\end{figure*}

\noindent\textbf{Contributions.} In this work, we propose the first random Subspace Zeroth-order (SubZero) optimization to tackle the challenges of high-dimensional LLM fine-tuning.  We introduce a low-rank  perturbation to estimate the gradient, specifically designed for LLM architecture, leading to reduced memory consumption and enhanced training performance. Our main contributions are as follows.

Firstly, we propose a layer-wise low-rank perturbation approach for gradient estimation, specifically designed for fine-tuning LLMs. In each layer, we generate a low-rank perturbation matrix by combining two column-orthogonal matrices with a Gaussian random matrix, which is then used for gradient estimation. Unlike traditional ZO methods like MeZO~\citep{malladi2023fine} which apply non-low-rank perturbations to the entire model, our approach significantly reduces the variance of gradient estimates and the angle error between the estimated gradient and its expectation, as respectively shown in Fig.~\ref{fig:motivation} (a) and (b). SubZero also improves upon random subspace ZO methods like S-RGF~\citep{nozawa2024zeroth} by using smaller and layer-specific low-rank perturbation matrices instead of a large and model-scale projection matrix, thus cutting memory and computational costs. Additionally, we introduce a lazy update strategy, generating perturbations periodically rather than iteratively, further reducing overhead. Besides, we successfully apply SubZero to four popular LLM fine-tuning schemes, highlighting the  compatibility of SubZero.

Secondly, we provide theoretical guarantees for SubZero. We first convert our gradient estimation into an equivalent formulation, highlighting the key differences between our approach and existing traditional ZO methods~\citep{malladi2023fine}, as well as random subspace ZO methods~\citep{nozawa2024zeroth}. Then, we prove that the gradient estimated by SubZero closely approximates the BP gradient, i.e., the ground-truth gradient,  and enjoys significantly lower gradient variance than traditional ZO methods like MeZO. Furthermore, we establish the theoretical convergence of SubZero when combined with the SGD optimizer.

Finally, experimental results demonstrate SubZero's superior performance and memory efficiency compared to other ZO approaches in both full-parameter tuning and parameter-efficient fine-tuning (PEFT) schemes, such as LoRA, prefix tuning, and prompt tuning. For instance, SubZero improves upon MeZO by 7.1\% on LLaMA-7B and by 3.2\% on OPT-1.3B under full-parameter tuning and prompt tuning, while maintaining nearly identical memory costs to MeZO.
\section{Related Work}
\noindent\textbf{Zeroth-Order Fine-Tuning.}
ZO optimizers utilize just two forward passes to estimate gradient without BP. \citet{malladi2023fine} first used ZO optimization to fine-tune LLMs, significantly lowering the GPU hours and memory usage to levels similar to inference, which offers a considerable advantage over FO optimizers. They demonstrated that LLM fine-tuning benefits from a well-structured loss landscape by introducing suitable task-specific prompt templates. Convergence theories for ZO optimization have been elaborated in both convex~\citep{nesterov2017random, jamieson2012query, duchi2015optimal} and non-convex settings~\citep{liu2018zeroth, ji2019improved}. However, these convergence rates typically increase linearly with the number of trainable parameters~\citep{nesterov2017random, jamieson2012query,duchi2015optimal, liu2018zeroth, ji2019improved}.

Recently, more work in ZO has focused on improving the convergence rates and reducing gradient estimation variance for LLM fine-tuning. Increasing batch size can diminish noise in ZO gradient estimation~\citep{gautamvariance, jiang2024zo}. Perturbing a subset of model parameters also lowers gradient variance. This approach induces sparse parameter perturbations through random and sparse pruning masks~\citep{liu2024sparse} or block-coordinate perturbations~\citep{zhang2024revisiting}. Additionally, some approaches tried to reduce trainable parameters through PEFT~\citep{malladi2023fine, zhang2024revisiting} and tensorized adapters~\citep{yang2024adazeta}. 

\noindent\textbf{Random Subspace Optimization.} To lessen dependence on dimensionality, some research utilizes random projections and low-dimensional perturbations in subspaces~\citep{nozawa2024zeroth, roberts2023direct, kozak2021stochastic}. However, these methods are hindered by the need to store a large projection matrix that increases with dimensionality, making it impractical for fine-tuning LLMs. 

\noindent\textbf{Memory-Efficient Fine-Tuning.} 
Fine-tuning generally employs FO optimizers like SGD~\citep{amari1993backpropagation} or Adam~\citep{kingma2014adam}. Various approaches have been developed to reduce the memory cost of BP, such as LoRA~\cite{hu2021lora,wang2024mlae}, gradient sparsification~\citep{sun2017meprop}, low-rank gradient projection~\citep{zhao2024galore}, and optimizer state quantization~\citep{dettmers2022bit,li2024memory}. Additional methods to conserve activation and weight memory during forward and backward passes include gradient checkpointing~\citep{chen2016training}, FlashAttention~\citep{dao2022flashattention}, QLoRA~\citep{dettmers2024qlora}, and LLM.int8()~\citep{dettmers2022gpt3}.

\section{Preliminaries}\label{sec:preliminaries}
Here we introduce the most popular ZO optimization approach and existing random subspace optimization methods.
 
\noindent\textbf{Notations.}
Let non-bold letter like $a$ and $A$  denote a scalar, a boldfaced lower-case letter like $\boldsymbol{w}$ denote a column vector, and a boldfaced upper-case letter like $\boldsymbol{W}$ denote a matrix. $\mathcal{N}(\boldsymbol{0}, \boldsymbol{I})$ is a multivariate normal distribution with a  zero mean vector and an identity covariance matrix. $ \text{vec}(\boldsymbol{W})$ denotes the vectorization of matrix $\boldsymbol{W}$ which reshapes $\boldsymbol{W}$ into a column vector by stacking the columns of $\boldsymbol{W}$ vertically. $\boldsymbol{A} \otimes \boldsymbol{B}$ is the Kronecker product of matrices $\boldsymbol{A}$ and $\boldsymbol{B}$. $\mathbb{E}[\vx]$ is the expected value of a random variable $\vx$. $\Var[\vx]$ is the variance of a random variable $\vx$. The $\ell_2$-norm of a vector $\vx$ is $\left \| \vx \right \| = \sqrt{\sum_{i=1}^{n}\vx_i^2 }$. The spectral norm of a matrix $\mA$ is $ \left \| \mA \right \|$. The Frobenius norm of a matrix $\bm{A}$ is $\|\bm{A}\|_F\!=\!\sqrt{\langle\bm{A}, \bm{A}\rangle}$. 
$\displaystyle C^{s,p}_L(\mathcal{S})$ denotes the class of $s$-th smooth and $p$-th $L$-smooth functions over the set $\mathcal{S}$.  ${\rm bdiag}(\bm{A}_1, \bm{A}_2, \cdots, \bm{A}_{l})$ is a block diagonal matrix with diagonal blocks $\bm{A}_1, \bm{A}_2, \cdots, \bm{A}_{l}$.

We are interested in fine-tuning large LLMs~\citep{ding2023parameter}. These models typically comprise multiple layers, with trainable parameter vectors represented as  $\boldsymbol{w}\!=\! \left [ \boldsymbol{w}_1^\mathsf{T}, \boldsymbol{w}_2^\mathsf{T}, \dots, \boldsymbol{w}_{l}^\mathsf{T} \right ]^\mathsf{T} \!\in\! \mathbb{R}^d$, where $\boldsymbol{w}_i$ denotes the flattened parameter vector from the $i$-th layer and $d$ is model parameter dimension. Then training these models involves optimizing the  problem: 
\begin{align}
	\min\nolimits_{\boldsymbol{w}} \mathcal{L}(\boldsymbol{w}),
\end{align}
where $\mathcal{L}(\cdot)$  denotes the loss function.

\noindent\textbf{Zeroth-Order Optimization.} ZO optimization is BP-free and estimates gradients via random perturbations. A classical gradient estimator is the simultaneous perturbation stochastic approximation (SPSA)~\citep{spall1992multivariate}, which is defined as
\begin{align}
	\widehat{\nabla} \mathcal{L}(\boldsymbol{w}; \mathcal{B})  = \frac{\mathcal{L}(\boldsymbol{w} + \varepsilon \boldsymbol{z}; \mathcal{B}) - \mathcal{L}(\boldsymbol{w} - \varepsilon \boldsymbol{z}; \mathcal{B})}{2 \varepsilon} \boldsymbol{z},
\label{eq:mezo_estimate_grad}
\end{align}
where  $\mathcal{L}(\boldsymbol{w} ; \mathcal{B})$  is the loss on a  minibatch $\mathcal{B}$ of size $B$ uniformly sampled from the training dataset $ \mathcal{D}$, 
$\boldsymbol{z} \in \mathbb{R}^d$ represents a random perturbation sampled from $\mathcal{N}(\boldsymbol{0}, \boldsymbol{I}_d)$, and $\varepsilon$ is the perturbation scale. 

The SPSA in Eqn.~\eqref{eq:mezo_estimate_grad} is an unbiased gradient estimator of the desired gradient $\nabla \E_{\boldsymbol{z}}[\mathcal{L}(\boldsymbol{w}+ \varepsilon\boldsymbol{z})] $~\citep{nesterov2017random}. It only requires two forward passes to estimate the gradient and eliminates  BP computation, greatly reducing  computation cost and GPU memory. With this estimated gradient, one can integrate with existing FO optimizers like SGD to develop corresponding ZO optimizers, e.g., ZO-SGD defined as:
\begin{align}
    \boldsymbol{w}^{t+1} = \boldsymbol{w}^{t} - \eta^t \widehat{\nabla} \mathcal{L}(\boldsymbol{w}^{t}; \mathcal{B}^t),
\end{align}
where $\eta^t$ is the learning rate at step $t$. In practice, MeZO~\citep{malladi2023fine} implements ZO-SGD via in-place operations and  uses a single random seed to  facilitate efficient perturbation regeneration,   greatly reducing memory overhead. 

\noindent\textbf{Random Subspace Optimization.}  
Recent theoretical work~\citep{nozawa2024zeroth, roberts2023direct} has explored using low-dimensional perturbations in random subspaces to reduce gradient variances and hence enhance convergence rates. The key to random subspace methods is the generation of the perturbation vector $\Tilde{\boldsymbol{z}}$ within a subspace spanned by  $\boldsymbol{P}$:  
\begin{align}
\label{eqn:subspace_spsa}
   \Tilde{\boldsymbol{z}} = \boldsymbol{P} \boldsymbol{z}, 
\end{align}
where $\boldsymbol{P} \in \mathbb{R}^{d \times q}$ is a random projection matrix with entries drawn from $\mathcal{N}\left(0, 1\right) $, $\boldsymbol{z} \in \mathbb{R}^{q}$ is a low-dimensional random perturbation vector sampled from $\mathcal{N}(\boldsymbol{0}, \boldsymbol{I}_q)$, and $q<d$ is the dimension of the subspace. Thus,  the gradient estimator in the subspace is given as follows:
\begin{align}
	\widehat{\nabla}\! \mathcal{L}(\boldsymbol{w}, \boldsymbol{P}; \mathcal{B}) \! = \!\frac{\mathcal{L}(\boldsymbol{w} \!+ \!\varepsilon \boldsymbol{P} \boldsymbol{z};\! \mathcal{B}) \!-\! \mathcal{L}(\boldsymbol{w} \!-\! \varepsilon \boldsymbol{P}\boldsymbol{z};\! \mathcal{B})}{2 \varepsilon} \boldsymbol{P}\boldsymbol{z}.
\label{eq:subspace_zo_estimate_grad}
\end{align}

LLMs have a large model size, and thus their training and fine-tuning parameters can be very high-dimensional. This results in an excessively large matrix $\boldsymbol{P}$ which is $q$ times larger than the model size $d$ in full-parameter tuning~\citep{aghajanyan2021intrinsic} and is also large in other fine-tuning schemes e.g., LoRA~\citep{hu2021lora}. Consequently, this approach significantly increases memory requirements and computational complexity. Therefore, it is crucial to develop an efficient subspace construction strategy with minimal memory consumption for LLM fine-tuning.

\section{Methodology}

Here we first elaborate on our SubZero, a powerful ZO  framework  for LLM fine-tuning. Then we present how to integrate SubZero into four  representative fine-tuning schemes.

\subsection{Random Subspace Optimization for LLM Fine-Tuning}\label{techniques}

Our intuition is that exploring update directions in a low-dimensional subspace may result in a reduced  variance of the estimated gradient~\citep{nozawa2024zeroth, roberts2023direct} compared to the estimation in the vanilla space as used in MeZO. Moreover, recent work indicates that BP gradients in LLM fine-tuning rapidly converge to a small subspace~\citep{zhang2023fine,malladi2023fine,zhao2024galore,hao2024flora}. Accordingly, we propose the random Subspace Zeroth-order (SubZero) optimization framework  tailored for LLM fine-tuning. This framework reduces gradient estimation variance, and minimizes the memory overhead  associated with gradient estimation, such as  the memory overhead caused by the projection matrix $\boldsymbol{P}$ in Eqn.~(\ref{eq:subspace_zo_estimate_grad}) used in~\citep{nozawa2024zeroth, roberts2023direct}.

\noindent\textbf{Layer-wise Random Subspace Perturbation.} LLMs  primarily consist of dense layers that perform matrix multiplication. We denote the trainable parameters of the $i$-th layer in matrix form as $\boldsymbol{W}_i \in \mathbb{R}^{m_i \times n_i}$. Then we will explain how to design its low-rank perturbation $\tilde{\boldsymbol{Z}}_i \in \mathbb{R}^{m_i \times n_i}$.  

We propose a low-rank perturbation strategy for model parameter matrix of each layer, contrasting with previous random subspace methods that focus on the entire model's parameters~\citep{nozawa2024zeroth, roberts2023direct}. At each step, we generate a low-dimensional random matrix $\boldsymbol{Z}_i \in \mathbb{R}^{r \times r} $, where $ r \ll \min\{m_i, n_i\} $,  and perform QR decomposition on two Gaussian random matrices with entries sampled from $\mathcal{N}(0,1)$ to create projection matrices $\boldsymbol{U}_i \in \mathbb{R}^{m_i \times r}$ and $\boldsymbol{V}_i \in \mathbb{R}^{n_i \times r}$ (see Algorithm~\ref{alg:getortho}). Both $\boldsymbol{U}_i$ and $\boldsymbol{V}_i$ are column-orthogonal matrices. Our experiments in Table~\ref{tab:abla_projtype} indicate that directly using Gaussian random projection matrices yields worse performance than using our designed column-orthogonal  matrices. Then we combine these three matrices  to yield a low-rank perturbation as follows:
\begin{align}\label{low_rank_pertubation}
    \tilde{\boldsymbol{Z}}_i = \boldsymbol{U}_i \boldsymbol{Z}_i \boldsymbol{V}_i^\mathsf{T}, 
\end{align}
where $\Tilde{\boldsymbol{Z}}_i$ is the perturbation matrix in a subspace spanned by  $\boldsymbol{U}_i$ and $\boldsymbol{V}_i$, and $\boldsymbol{Z}_i$ represents the low-dimensional random perturbation matrix with entries sampled from $\mathcal{N}(0,1)$. The projection matrices for SubZero can only be derived using the QR decomposition of the weights, activations, ZO gradients, or random matrix, without increasing memory overhead. As shown in Table~\ref{tab:proj_choice}, our proposed random matrix achieves the best performance. While the weight and ZO gradient matrices are feasible, their performance drops significantly. The activation matrix is ineffective due to its batch size dependency, requiring more sophisticated handling. Detailed experimental setups are in Appendix~\ref{sec: hyperparam}.

Let the model consist of $l$ layers, with the parameter matrix set defined as $\mathcal{W}=\{\boldsymbol{W}_{i}\}_{i=1}^{l}$ and the perturbation matrix set as $\tilde{\mathcal{Z}}=\{\Tilde{\boldsymbol{Z}}_i\}_{i=1}^{l}$.  
Similar to Eqns.~\eqref{eq:mezo_estimate_grad} and~\eqref{eq:subspace_zo_estimate_grad}, we  compute the loss  difference: 
\begin{equation}
	\rho = \frac{\mathcal{L}(\mathcal{W} + \varepsilon \tilde{\mathcal{Z}}; \mathcal{B} ) - \mathcal{L}(\mathcal{W} - \varepsilon \tilde{\mathcal{Z}} ; \mathcal{B})}{2 \varepsilon} \label{eq:cal_rho}.
\end{equation}
Note that multiplying a set by a scalar means that the scalar is multiplied by each element in the set. The addition of two sets means that the corresponding elements are added. This is only for mathematical expression, and $\rho$ in Eqn.~\eqref{eq:cal_rho} can be calculated by two forward passes through all the layers in practice. Then we obtain the gradient estimate for the $i$-th layer as 
\begin{equation} \label{eq:gradient_estimate}
\widehat{\nabla} \mathcal{L}(\boldsymbol{W}_i; \mathcal{B})= \rho \tilde{\boldsymbol{Z}}_i = \rho \boldsymbol{U}_i \boldsymbol{Z}_i \boldsymbol{V}_i^\mathsf{T}.
\end{equation}
In Sec.~\ref{sec:theroy}, we analyze the effectiveness of this new gradient estimation~(\ref{eq:gradient_estimate}).  Specifically, Theorem~\ref{thm:oracles-mean} proves the close distance between our gradient estimate~\eqref{eq:gradient_estimate} and the vanilla gradient computed  by BP in FO methods, while Theorem~\ref{thm:oracles-quadratic} shows smaller variance and angle error of our gradient estimate in Eqn.~\eqref{eq:gradient_estimate} compared to the gradient estimate~\eqref{eq:mezo_estimate_grad} in  MeZO~\citep{malladi2023fine}. See more theoretical details in Sec.~\ref{sec:theroy}.

Then, one can use estimated gradient in~\eqref{eq:gradient_estimate} to replace the gradient in any FO
optimizer such as SGD:
\begin{equation}\label{update}
	\boldsymbol{W}_{i}^{t+1}\! =\!  \boldsymbol{W}_i^{t} \!-\! \eta^t  \widehat{\nabla} \mathcal{L}(\boldsymbol{W}_i^t;\! \mathcal{B}^t)\! =\! \boldsymbol{W}_i^{t} \! -\! \eta^t \rho^{t} \boldsymbol{U}_i^t \boldsymbol{Z}_i^t {\boldsymbol{V}_i^t}^\mathsf{T}.
\end{equation}
Here we choose SGD as the default optimizer of SubZero. Theorem~\ref{thm:convergence}  in Sec.~\ref{sec:theroy} guarantees  the  convergence  of SubZero with SGD as basic optimizer and gives its convergence rate.  
The choice of FO optimizers is orthogonal to ZO optimization. Also, some empirical work indicates that adaptive optimizers like Adam~\citep{kingma2014adam} do not necessarily enhance  convergence of ZO approaches  during LLM fine-tuning~\citep{zhang2024revisiting, guo2024zeroth}. So the combination of SubZero and Adam is included in Appendix~\ref{sec:add_results} due to the limited space. We apply the primitive ZO approach. There are other ZO optimizers that utilize stochastic momentum~\citep{jiang2024zo} and second-order information~\citep{zhao2024second} to facilitate faster convergence.  While SubZero can be adapted to these ZO optimizers,  we leave a comprehensive evaluation of these approaches for future work.


\begin{table}[t] 
    \vspace{-0.3em}   
    \begin{minipage}{0.27\textwidth}
        \centering
		\begin{small}
        \caption{Projection matrix generation for SubZero in full-parameter tuning with OPT-1.3B and LLaMA2-7B on SST-2 and OPT-13B on RTE.}
        \vspace{-1em}
        \label{tab:proj_choice}
        \setlength{\tabcolsep}{2pt}
        \renewcommand{\arraystretch}{0.95} 
        \resizebox{0.7\textwidth}{!}{
            \begin{tabular}{l|ccc}
                \toprule
                Matrix & 1.3B &  7B & 13B \\
                \midrule
                Weight & 91.5 & 91.7 & 65.3\\
                Activation & 51.5& 52.9 & 53.1 \\
                ZO Gradient & 89.6 & 92.0 & 67.5\\
                Random & \textbf{93.4}& \textbf{94.5}& \textbf{74.0} \\
                \bottomrule
            \end{tabular}
        }
		\end{small}
    \end{minipage}
    \hspace{0.8em}
    \begin{minipage}{0.18\textwidth}
    	\centering
    	\begin{small}
    		\caption{Memory cost in full-parameter tuning   with RoBERTa-large on SST-2. }
    		\vspace{-1em}
    		\label{tab:memorycompare}
    		\setlength{\tabcolsep}{3pt}
    		\renewcommand{\arraystretch}{1} 
    		\resizebox{0.9\textwidth}{!}{
    			\begin{tabular}{l|c}
    				\toprule
    				Method & Mem.(GB) \\
    				\midrule
    				SGD & ~6.063  \\
    				MeZO \citep{malladi2023fine}& ~2.683  \\            
    				S-RGF \citep{nozawa2024zeroth}& 23.845 \\
    				SubZero & ~2.690 \\
    				\bottomrule
    			\end{tabular}
    		}
    	\end{small}
    \end{minipage}
    \vspace{-0.5cm}
\end{table}

We compare the memory overhead of SubZero with the existing random subspace method \mbox{S-RGF}~\citep{nozawa2024zeroth} using identical experimental settings, including layer-wise perturbation and matching subspace dimension, with all methods utilizing the SGD optimizer. As shown in Table~\ref{tab:memorycompare}, S-RGF's memory usage is roughly four times greater than SGD and 8.8 times that of MeZO~\citep{malladi2023fine}, while our SubZero's memory usage is comparable to MeZO. See more experimental comparisons on OPT-13B in Table~\ref{tab:memoryusage} of Appendix~\ref{sec:add_results}.

\noindent\textbf{Lazy Low-rank Subspace Update.} According to Eqn.~\eqref{update}, at the $t$-th step, the gradient estimate of the parameter matrix in the  $i$-th layer, $\widehat{\nabla} \mathcal{L}(\boldsymbol{W}_i^t; \mathcal{B}^t)$, lies within a subspace defined by the  projection matrices $\boldsymbol{U}_{i}^{t}$ and $\boldsymbol{V}_{i}^{t}$. Specifically,  $\boldsymbol{U}_{i}^{t}$  spans the column subspace, while  $\boldsymbol{V}_{i}^{t}$ determines the row subspace, with both matrices generated iteratively, leading to extra computational overhead to LLM fine-tuning.  

However, for LLM fine-tuning, enhancing the computational efficiency and the accuracy of gradient subspace approximation is crucial.  An excessively short update interval for  $\boldsymbol{U}_{i}$   and  $\boldsymbol{V}_{i}$, such as generating them iteratively, can incur high computational costs and limit exploration of the gradient subspace they established. Conversely, a long interval may result in inaccuracies in subspace approximation and fail to capture the evolving nature of the gradient subspace.  Accordingly, we propose a lazy subspace update strategy that periodically regenerates the projection matrices $\boldsymbol{U}_{i}$   and  $\boldsymbol{V}_{i}$. Specifically, these matrices are generated at the first step of every $F>1$ training steps and remain unchanged for the subsequent $F-1$ steps (see lines 4-7 in Algorithm~\ref{alg:ss_mezo}). We utilize QR decomposition on two different random matrices for generating the column-orthogonal matrices  $\boldsymbol{U}_{i}$   and  $\boldsymbol{V}_{i}$, as summarized in Algorithm~\ref{alg:getortho}. This lazy subspace update strategy is both efficient and effective in all our experiments.

\noindent 
\begin{minipage}[c]{0.47\textwidth}
	\vspace{-1em}
	\begin{algorithm}[H]	
		\caption{GenerateProjMatrix$(m, n, r)$}
		\label{alg:getortho}	
		\begin{algorithmic}[1]
			\REQUIRE{size of parameter matrix $m\times n$,  rank $r$.
			}			
			\STATE Generate random matrices $\boldsymbol{R}_1 \in \mathbb{R}^{m \times r} $ and  $\boldsymbol{R}_2 \in \mathbb{R}^{n \times r}$ whose entries are sampled from $\mathcal{N}(0,1)$ \\
			\STATE $\boldsymbol{U}, \_ \leftarrow$ QR\_Decomposition$(\boldsymbol{R}_1)$ 
			\STATE $\boldsymbol{V}, \_ \leftarrow$ QR\_Decomposition$(\boldsymbol{R}_2)$ 
			
			\RETURN $\boldsymbol{U}$, $\boldsymbol{V}$			
		\end{algorithmic}	 	
	\end{algorithm}
	\vspace{-2em}  
	\begin{algorithm}[H]
		\caption{PerturbParams$(\mathcal{W}, \mathcal{U}, \mathcal{V}
			, r, \varepsilon, s)$ }
		\label{alg:subspaceperturb}
		\begin{algorithmic}[1]
			\REQUIRE{
				model parameter set $\mathcal{W}$, projection matrix sets $\mathcal{U} $ and $\mathcal{V}$, rank $r$, perturbation scale $\varepsilon$, seed $s$. \\
			}
			
			\STATE Reset random number generator with seed $s$ \\
			\FOR{$i = 1, 2, \dots, l$}
			\STATE Generate the perturbation matrix $\boldsymbol{Z}_i  \in \mathbb{R}^{r \times r} $ whose entries are sampled from $\mathcal{N}(0,1)$ \\
			\STATE $\boldsymbol{W}_i \leftarrow \boldsymbol{W}_i + \varepsilon \boldsymbol{U}_i \boldsymbol{Z}_i \boldsymbol{V}_i^\mathsf{T} $\\
			\ENDFOR
			\RETURN $\mathcal{W}$			
		\end{algorithmic}
	\end{algorithm}
	\vspace{-2em} 	
	\begin{algorithm}[H]	
		\caption{SubZero}
		\label{alg:ss_mezo}
		\begin{algorithmic}[1]
			\REQUIRE{
				parameter matrix in the $i$-th layer $\boldsymbol{W}_i \in \mathbb{R}^{m_i \times n_i}, i=1,2,\dots, l$, loss $\mathcal{L}$, step budget $T$, perturbation scale $\varepsilon$, learning rate schedule $\{\eta^t\}$, subspace change frequency $F$, rank $r$. \\
			}           
			
			\FOR{$t = 0, 1, \dots, T-1$}
			\STATE Sample a minbatch $\mathcal{B}^t \subset \mathcal{D}$ and a random seed  $s^t$\\
			\FOR{$i = 1, 2, \dots, l$}
			\IF{$t\mod F  \equiv 0$ } 
			\STATE $\boldsymbol{U}_{i}^{t}$,$\boldsymbol{V}_{i}^{t} \leftarrow $ GenerateProjMatrix$(m_i, n_i, r) $
			\ELSE 
			\STATE $\boldsymbol{U}_{i}^{t} \leftarrow \boldsymbol{U}_{i}^{t-1}$, $\boldsymbol{V}_{i}^{t} \leftarrow \boldsymbol{V}_{i}^{t-1}$
			\ENDIF
			\ENDFOR            
			\STATE // Note that $\mathcal{W}^t=\{\boldsymbol{W}_{i}^{t}\}_{i=1}^{l}$, $\mathcal{U}^t=\{\boldsymbol{U}_{i}^{t}\}_{i=1}^{l}$, \\ ~~~$\mathcal{V}^t=\{\boldsymbol{V}_{i}^{t}\}_{i=1}^{l}$                  
			\STATE 	$\mathcal{W}^t \leftarrow $ PerturbParams $(\mathcal{W}^t, \mathcal{U}^t$, $\mathcal{V}^t, r,  \varepsilon, s^t)$,    \\  $\ell_{+}^t \leftarrow \mathcal{L}(\mathcal{W}^t; \mathcal{B}^t)$\\
			\STATE 	$\mathcal{W}^t \leftarrow $ PerturbParams $(\mathcal{W}^t, \mathcal{U}^t$, $\mathcal{V}^t, r, -2 \varepsilon, s^t)$, \\ $\ell_{-}^t \leftarrow \mathcal{L}(\mathcal{W}^t; \mathcal{B}^t)$\\
			\STATE 	$\mathcal{W}^t \leftarrow $ PerturbParams $(\mathcal{W}^t, \mathcal{U}^t$, $\mathcal{V}^t, r, \varepsilon, s^t)$\\		
			\STATE $\rho^t$ $\leftarrow\left(\ell_{+}^t-\ell_{-}^t\right) /(2 \varepsilon)$   
			\STATE Reset random number generator with seed $s^t$
			\FOR{$i = 1, 2, \dots, l$}                   
			\STATE Regenerate the perturbation matrix $\boldsymbol{Z}_{i}^{t} \in \mathbb{R}^{r \times r}$ whose entries are sampled from $\mathcal{N}(0,1)$
			\STATE $\boldsymbol{W}_i^{t+1} \leftarrow \boldsymbol{W}_i^t - \eta^t \rho^t  \left(\boldsymbol{U}_{i}^{t} \boldsymbol{Z}_{i}^{t} {\boldsymbol{V}_{i}^{t}}^\mathsf{T}\right) $\\                
			\ENDFOR            
			\ENDFOR
			\RETURN $\mathcal{W}^{t+1}$  
		\end{algorithmic}
	\end{algorithm}
	\vspace{-1em}
\end{minipage}

 SubZero maintains just three small matrices per layer: a perturbation matrix $\boldsymbol{Z}_i \in \mathbb{R}^{r \times r}$, and two column-orthogonal matrices $\boldsymbol{U}_i \in \mathbb{R}^{m_i \times r}$ and $\boldsymbol{V}_i \in \mathbb{R}^{n_i \times r}$. This design enhances memory efficiency, as $r$ is generally much smaller than the size of the corresponding parameter matrix $\boldsymbol{W}_i \in \mathbb{R}^{m_i \times n_i}$ (i.e., $r \ll \min\{m_i, n_i\}$). Moreover, we employ in-place operations and per-layer parameter updates to estimate gradients and update parameters in parallel (see Appendix~\ref{app:details}). Consequently, SubZero uses significantly less GPU memory than previous methods while achieving similar or better performance. For example, fine-tuning OPT-1.3B~\citep{zhang2022opt} on SST-2~\citep{socher-etal-2013-recursive} using SGD (without momentum) in full-parameter scheme as shown in Table~\ref{tab:zo_peft}, SubZero requires only 6.8GB GPU memory, compared to 11.5GB for SGD, yielding a $1.6 \times$ improvement in memory efficiency, similar as illustrated in Fig.~\ref{fig:motivation} (d).   

Now we are ready to summarize the overall algorithm of SubZero in Algorithm~\ref{alg:ss_mezo}. Each training step consists of three sequential phases. First, it obtains the projection matrices $\boldsymbol{U}_{i}^{t}$   and  $\boldsymbol{V}_{i}^{t}$ using Algorithm \ref{alg:getortho} or directly adopts previous ones.  Next, it computes the loss value difference $\rho$ with Eqn.~\eqref{eq:cal_rho} by applying Algorithm~\ref{alg:subspaceperturb} to perturb all parameter matrices. Finally, SubZero updates all parameter matrices layer by layer, following Eqn.~\eqref{update}.

\subsection{Integration into Fine-Tuning Schemes}\label{PEFT}
We describe the integration of SubZero into full-parameter tuning~\citep{aghajanyan2021intrinsic} and three prominent PEFT schemes: LoRA~\citep{hu2021lora}, prefix tuning~\citep{li2021prefix}, and prompt tuning~\citep{lester2021power}. Typically, SubZero can be easily incorporated into these fine-tuning schemes. However, it encounters a challenge with extremely non-square parameter matrices, which have far more rows than columns or vice versa. This issue is  particularly prevalent in LoRA, which employs two low-rank matrices $\boldsymbol{A}_i\in\mathbb{R}^{m_i\times k }$ and $\boldsymbol{B}_i\in \mathbb{R}^{k\times n_i}$  to approximate a full matrix $\boldsymbol{W}_i'\in\mathbb{R}^{m_i\times n_i}$, with  $k \ll \min\{m_i,n_i\}$, e.g., $k=8$ while $\min\{m_i,n_i\}=2048$ used in~\citep{zhang2024revisiting}. Consequently, it is impossible to find a smaller rank $r\ll k$ to compute the gradient estimates of $\boldsymbol{A}_i$ and $\boldsymbol{B}_i$ using Eqn.~\eqref{low_rank_pertubation}, imposing a challenge when applying SubZero to  this scenario.

To overcome this limitation, we propose a reshaping strategy that transforms the original non-square matrix into an approximate square matrix. For instance, we reshape  $\boldsymbol{A}_i\in\mathbb{R}^{m_i\times k}$ into $\boldsymbol{A}_i' \in\mathbb{R}^{m_i'\times k'}$ such that $m_ik=m_i'k'$  and $m_i'$ is close to $k'$.  This reshaping allows us to  apply Eqn.~\eqref{low_rank_pertubation} to find a low-rank perturbation with rank $r$ significantly smaller than $\min\{m_i', k'\}$, demonstrating the applicability of SubZero in the scenario. Table~\ref{tab:abla_reshape} in Sec.~\ref{sec:ablation} shows the effectiveness of this reshaping strategy.
\section{Theoretical Analysis}\label{sec:theroy}

In this section, we theoretically analyze why SubZero can reduce the variance of gradient estimates and hence accelerate convergence.
Before the analysis, we first define some necessary notations: 
\begin{align}
	\label{eq:def_P}\mP &= {\rm bdiag} (\boldsymbol{V}_1 \otimes \boldsymbol{U}_1, 
	\cdots, \boldsymbol{V}_l \otimes \boldsymbol{U}_l )\!,  \\
	\vz &= [{\rm vec}(\mZ_1)^{\mathsf{T}}, 
	 \dots, {\rm vec}(\mZ_l)^{\mathsf{T}}]^{\mathsf{T}}\!, \\
	\tilde{\vz} &= [{\rm vec}(\tilde{\mZ}_1)^{\mathsf{T}}, 
	 \dots, {\rm vec}(\tilde{\mZ}_l)^{\mathsf{T}}]^{\mathsf{T}}\!.
\end{align}
Then we first state the main theoretical results on our gradient estimation in Eqn.~\eqref{eq:gradient_estimate}. 

\begin{restatable}{theorem}{theoremmean}
	\label{thm:oracles-mean}
	For the gradient estimation in Eqn.~\eqref{eq:gradient_estimate}, the following two properties hold.  \\
	\textbf{a)} By using gradient estimation in~\eqref{eq:gradient_estimate}, our estimated gradient $\hat{g}_{\varepsilon}(\vx, \mP, \boldsymbol{z})$  is equivalent to 
	\begin{equation}\label{safsad}
		\hat{g}_{\varepsilon}(\vx, \mP, \boldsymbol{z})=\frac{f(\vx+\varepsilon\mP\boldsymbol{z})-f(\vx-\varepsilon\mP\boldsymbol{z})}{2\varepsilon}\mP\boldsymbol{z},
	\end{equation}
	 where   $\boldsymbol{z} \sim \mathcal{N} (\boldsymbol{0} , \mI_q)$, $\varepsilon > 0$, $\mP\in\mathbb{R}^{d\times q}$ satisfies $\mP^{\mathsf{T}}\mP=\mI_q$ with $d=\sum_{i=1}^l m_i n_i$   and $q=lr^2$.\\
	\textbf{b) } 
	Let $\boldsymbol{z} \sim \mathcal{N} (\boldsymbol{0} , \mI_q)$, and $f\in C^{2,2}_{L_2}(\R^d)$. Then we have 
	\begin{align*}
	\Phi(\vx ) \!=\!\| \mathbb{E}_{\boldsymbol{z}}[ \hat{g}_{\varepsilon}(\vx, \mP, \boldsymbol{z}) ] \!-\! \mP\mP^{\mathsf{T}}\nabla f(\vx) \|_2 \!\le\! \frac{\varepsilon^2}{6}L_2(q+4)^2.
	\end{align*}
\end{restatable}
See its proof in Appendix~\ref{appendix: proofs}. Theorem~\ref{thm:oracles-mean} (a) provides the equivalent form~\eqref{safsad} of our gradient estimation~\eqref{eq:gradient_estimate}. By comparing this with the gradient estimation~\eqref{eq:subspace_zo_estimate_grad} in random subspace optimization~\citep{nozawa2024zeroth, roberts2023direct}, we observe significant differences. First, our gradient estimation~\eqref{safsad} accounts for the layer-wise structure of the network, requiring the projection matrix \(\mP\) to be block-diagonal, whereas in random subspace optimization, \(\mP\) is not. Additionally, our method introduces a layer-wise low-rank perturbation matrix, reflected by the block-diagonal structure of \(\mP\), with lazy updates to the column and row spaces defined by \(\mU_i\) and \(\mV_i\). In contrast, random subspace optimization simply requires \(\mP\) to be random. These distinctions highlight the key differences between our gradient estimation  and existing methods in random subspace optimization.

 Theorem~\ref{thm:oracles-mean} (b) guarantees that the distance \( \Phi(\vx ) \) between the expected gradient estimate and the BP gradient in the subspace spanned by \(\mP\) is small. Moreover, by setting \(\varepsilon = \frac{1}{q+4}\), the distance \( \Phi(\vx ) \) is bounded by a constant \( L_2 / 6 \), independent of the parameter dimension \(d\). This implies that the error in our gradient estimation does not scale with the extremely high parameter dimensions of LLMs, providing accurate gradient estimation—crucial for optimizing LLMs.

Next, we utilize a strictly convex quadratic loss to further analyze our gradient estimation in Eqn.~\eqref{safsad}. This choice is motivated by the fact that, after pretraining, the LLM parameters tend to converge toward a local minimum within a local basin, which can be well-approximated by a quadratic loss~\citep{neyshabur2020being}.
 
\begin{restatable}{theorem}{theoremquadratic}
	\label{thm:oracles-quadratic}
	Let $f(\vx)=\vx^{\mathsf{T}}\mH\vx$ and $\boldsymbol{z} \sim \mathcal{N} (\boldsymbol{0} , \mI_q)$, where $\mH\in\mathbb{R}^{d\times d}$ is positive definite. We have
	\begin{align}
		&\label{eq:quad_dist}
		\mathbb{E}_{\boldsymbol{z}}[ \hat{g}_{\varepsilon}(\vx, \mP, \boldsymbol{z}) ]  = \mP\mP^{\mathsf{T}}\nabla f(\vx) ,\\
	    &\label{eq:quad_var}	
		\mathbb{E}_{\boldsymbol{z}}[ \| \hat{g}_{\varepsilon}(\vx, \mP, \boldsymbol{z}) \|^2 ] = (q+2) \| \mP^{\mathsf{T}}\nabla 	f(\vx) \|^2 ,\\
        &\label{eq:quad_cosine_sim}
		\mathbb{E}_{\boldsymbol{z}} \left[ \frac{\langle \nabla f(\vx), \hat{g}_{\varepsilon}(\vx, \mP, \boldsymbol{z}) 	\rangle^2}{\| \mP^{\mathsf{T}}\nabla f(\vx) \|^2 \| \hat{g}_{\varepsilon}(\vx, \mP, \boldsymbol{z}) \|^2} \right]  = \frac{1}{q} .
	\end{align}
	
\end{restatable}

See its proof in Appendix~\ref{appendix: proofs}. Theorem~\ref{thm:oracles-quadratic} demonstrates several advantageous properties of our gradient estimation on the quadratic function. First, Eqn.~\eqref{eq:quad_dist} establishes the equivalence between the expected gradient estimation and the BP gradient within the subspace spanned by our projection matrix \(\mP\). Second, Eqn.~\eqref{eq:quad_var} shows that, in this subspace, the variance of the gradient estimation scales linearly with the subspace dimension \(q\). In contrast, the variance of gradient estimation~\eqref{eq:mezo_estimate_grad} in MeZO depends linearly on the model's parameter dimension \(d\), which is significantly larger than \(q\). Finally, Eqn.~\eqref{eq:quad_cosine_sim} reveals that the expected cosine similarity between our estimated gradient and the BP gradient within the subspace depends only on the subspace dimension \(q \ll d\), indicating that our gradient estimation provides a highly accurate parameter update direction.

Building upon the above results, we can prove the convergence of our SubZero. 
\begin{restatable}{theorem}{theoremconvergence}
	\label{thm:convergence}	
	Let $f \in C^{1,1}_{L_1}(\mathbb{R}^d)$ be a non-convex function bounded below by $f^*$. Suppose $\mathcal{E}_{k} = (\boldsymbol{z}_0, \boldsymbol{z}_1, \cdots, \boldsymbol{z}_{k})$ with $\boldsymbol{z}_{k} \sim \mathcal{N}(\boldsymbol{0}, \mI_q)$, and let $\mathcal{P}_j = (\mP_0, \mP_1, \cdots, \mP_{j})$, where $\mP_j$ is defined in~\eqref{eq:def_P} with a fixed update frequency $F$. Then, the sequence $\{\vx_{k}\}_{k>0}$ generated by Algorithm~\ref{alg:ss_mezo} satisfies:
	\begin{align*}
		\frac{1}{T} \sum_{k=0}^{T - 1} \mathbb{E}_{\mathcal{E}_{k}, \mathcal{P}_{\left \lfloor k / F \right \rfloor }} \left[\| \nabla f(\vx_{k}) \|^{2}\right] \leq \epsilon
	\end{align*}
	with $T = \mathcal{O}\left(\frac{d}{\epsilon}\right)$ if the perturbation scale satisfies $\varepsilon \leq \mathcal{O} \left(\frac{\epsilon^{1/2}}{q^{3/2} d^{1/2} L_1^{3/2}}\right)$. Here $T = KF$, where $K$ denotes the total number of subspace updates.
\end{restatable}
See its proof in Appendix~\ref{appendix: proofs}. Theorem~\ref{thm:convergence} guarantees the convergence of our SubZero when the projection matrix $\boldsymbol{P}$ is updated at a fixed frequency $F$.

\section{Experiments}\label{sec:experiment}

\begin{table*}[t]
	\caption{Performance of fine-tuning OPT-13B on SuperGLUE with various experimental settings (with 1000 examples). AVG: average relative percentage difference with MeZO of all tasks.}
\vspace{-1em}	 
\label{tab:opt13b_performance}
\vspace{-0.7em}
\begin{center}
	\setlength{\tabcolsep}{4pt}
	\renewcommand{\arraystretch}{0.9} 
	\resizebox{1.0\linewidth}{!}{
		\begin{tabular}{l|ccccccccccc|c}
			\toprule
			Task type & \multicolumn{7}{c}{------------------------ classification ------------------------} & \multicolumn{2}{c}{-- multiple choice --} & \multicolumn{2}{c|}{--- generation ---} &\\ 
			Task & 
			\multicolumn{1}{c}{\bf SST-2} & 
			\multicolumn{1}{c}{\bf RTE} & 
			\multicolumn{1}{c}{\bf CB} & 
			\multicolumn{1}{c}{\bf BoolQ} & 
			\multicolumn{1}{c}{\bf WSC} & 
			\multicolumn{1}{c}{\bf WIC} & 
			\multicolumn{1}{c}{\bf MultiRC} & 
			\multicolumn{1}{c}{\bf COPA} & 
			\multicolumn{1}{c}{\bf ReCoRD} & 
			\multicolumn{1}{c}{\bf SQuAD} & 
			\multicolumn{1}{c|}{\bf DROP} &
			\multicolumn{1}{c}{\bf AVG.} \\
			\midrule
			SGD(FT) & 94.9 & 82.3& 85.7&78.4 & 65.3& 65.8&74.2 & 90.0& 82.4& 88.0& 35.5& -\\                      
			\midrule
			Zero-shot&58.8&59.6&46.4&59.0&38.5&55.0&46.9&80.0&81.2&46.2&14.6 & -\\
			ICL&87.0&62.1&57.1&66.9&39.4&50.5&53.1&87.0&82.5&75.9&29.6 & -\\
			LP&93.4&68.6&67.9&59.3&63.5&60.2&63.5&55.0&27.1&3.7&11.1 &-\\
			\midrule
			MeZO(FT) & 92.1  & 71.5 & 71.4 &74.4 &61.5 &60.0 &60.1 & 87.0 & 82.0&84.2 & 31.2 &0\%\\
			ZO-AdaMU(FT) & 92.1 & 72.9&  67.9 & 73.0 & 61.5 & 60.7 & 63.0 & \bf{89.0} & \bf{83.0} & 82.4 & \bf{32.0} &0.46\%\\
			S-MeZO(FT) & 92.3 & \bf{76.9} & \bf{75.0} & \bf{76.5}& 61.1& 58.2& \bf{63.3}&87.0 &71.2 &77.9 & 31.9 & -0.10\%\\
			HiZOO(FT) & 91.3 & 69.3 & 69.4 & 67.3 & 63.5 & 59.4 & 55.5 & 88.0 & 81.4 & 81.9 & 31.3 & -2.15\%\\
			LOZO(FT) & \bf{92.9} & 73.6 & 71.4 & 70.7 & 63.5 & 60.2 & 60.3 & 87.0 & 81.7 & \bf{84.5} & 30.7 & 0.10\%\\
			SubZero(FT) & 92.1 & 74.0& 73.2 & 75.3& \bf{65.4} & \bf{60.8} & 61.0& 88.0  & 82.3 & \bf{84.5} & \bf{32.0} &\bf{1.89\%}\\
			\midrule
			MeZO(LoRA) & 92.2  & 74.4& 69.6&75.2 &64.4 &59.7 &58.2 &87.0 & 82.0 & 82.9 & 31.0& 0\%\\
			ZO-AdaMU(LoRA) & 88.0 & 72.0 & 71.6 & 72.6 & 60.1 & 56.4 & 58.9 & 88.0 & \bf{83.2} & 76.8 & \bf{32.4} & -1.78\%\\
			S-MeZO(LoRA) & 90.8 & 62.2 & \bf{75.0} & 72.9 &51.9 &55.8&  56.4& 86.0 & 69.9 & 76.4 & 31.7 & -5.79\%\\
			HiZOO(LoRA)  &90.6 &67.5 &69.6 &70.5 &63.5 &60.2 &60.2 &87.0 &81.9 & \bf{83.8}& 31.2& -1.16\%\\
			SubZero(LoRA) & \bf{93.8} & \bf{75.5} &  71.4& \bf{76.1}& \bf{65.4} & \bf{60.3} & \bf{60.3} & \bf{89.0} & 81.9 & 83.7 & 31.3 & \bf{1.57\%}\\
			\bottomrule
	\end{tabular}}
\end{center}
\vspace{-1.5em}
\end{table*}

\begin{figure*}[t]	
	\centering	
	\hfill
	\begin{subfigure}{0.24\textwidth}
		\includegraphics[width=\textwidth]{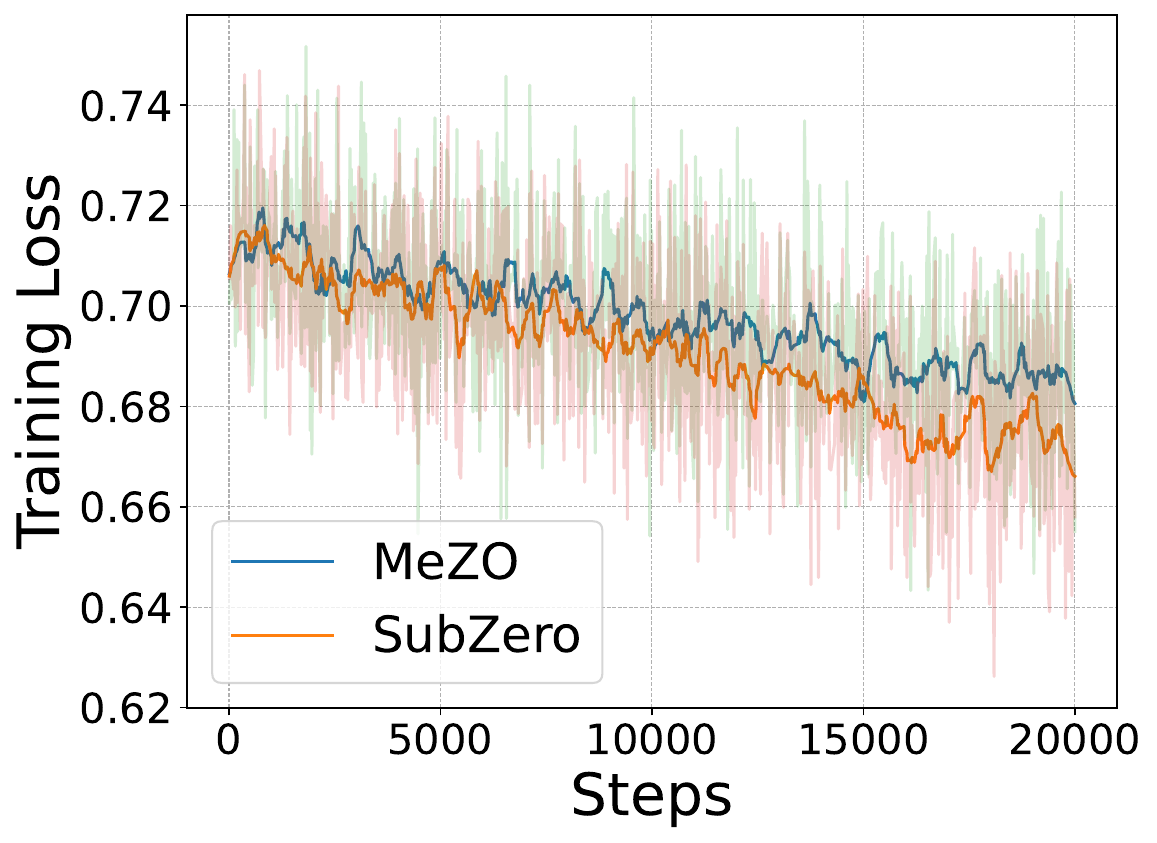}
		\caption{WIC: OPT-13B, FT}		
	\end{subfigure}
	\hspace{0.01em}
	\begin{subfigure}{0.24\textwidth}
		\includegraphics[width=\textwidth]{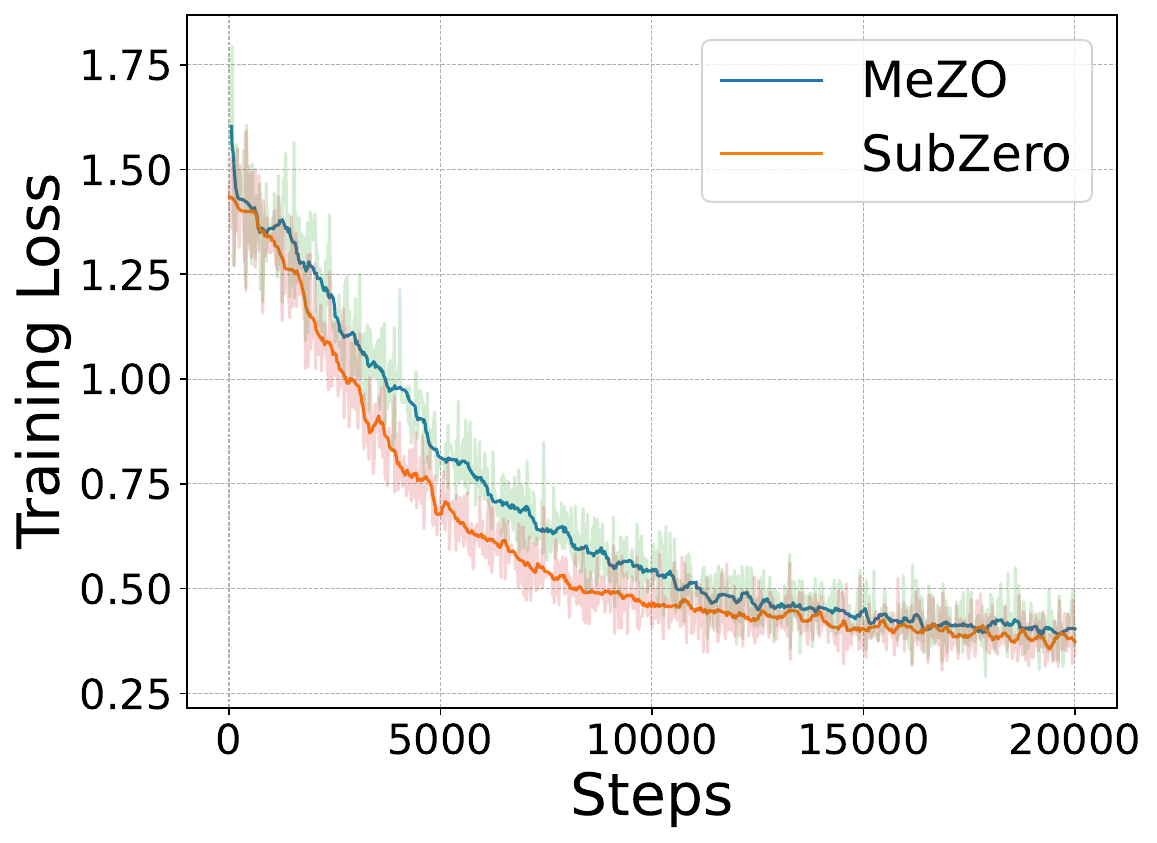}
		\caption{SQuAD: OPT-13B, LoRA}	
	\end{subfigure}
	\hspace{0.01em}
	\begin{subfigure}{0.235\textwidth}
		\includegraphics[width=\textwidth]{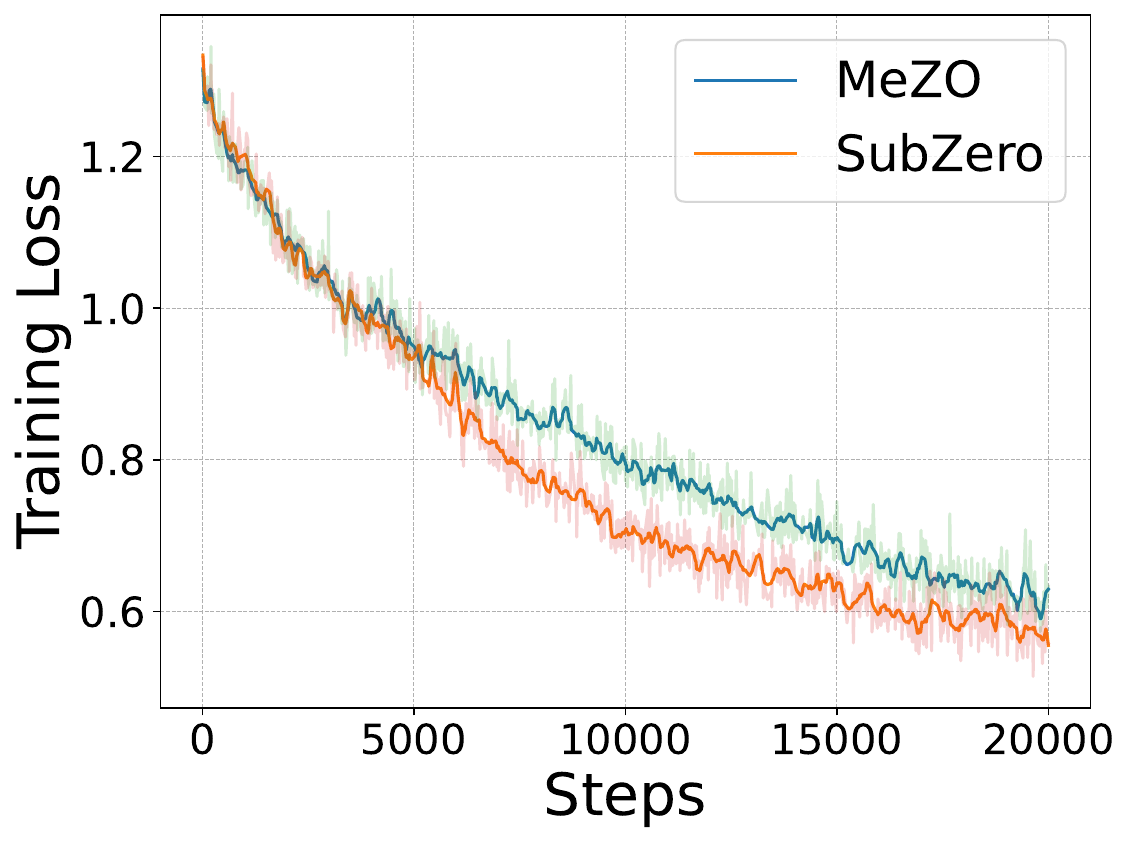}
		\caption{CB: LLaMA2-7B, Prompt}
	\end{subfigure}	
	\hspace{0.01em}
	\begin{subfigure}{0.244\textwidth}
		\includegraphics[width=\textwidth]{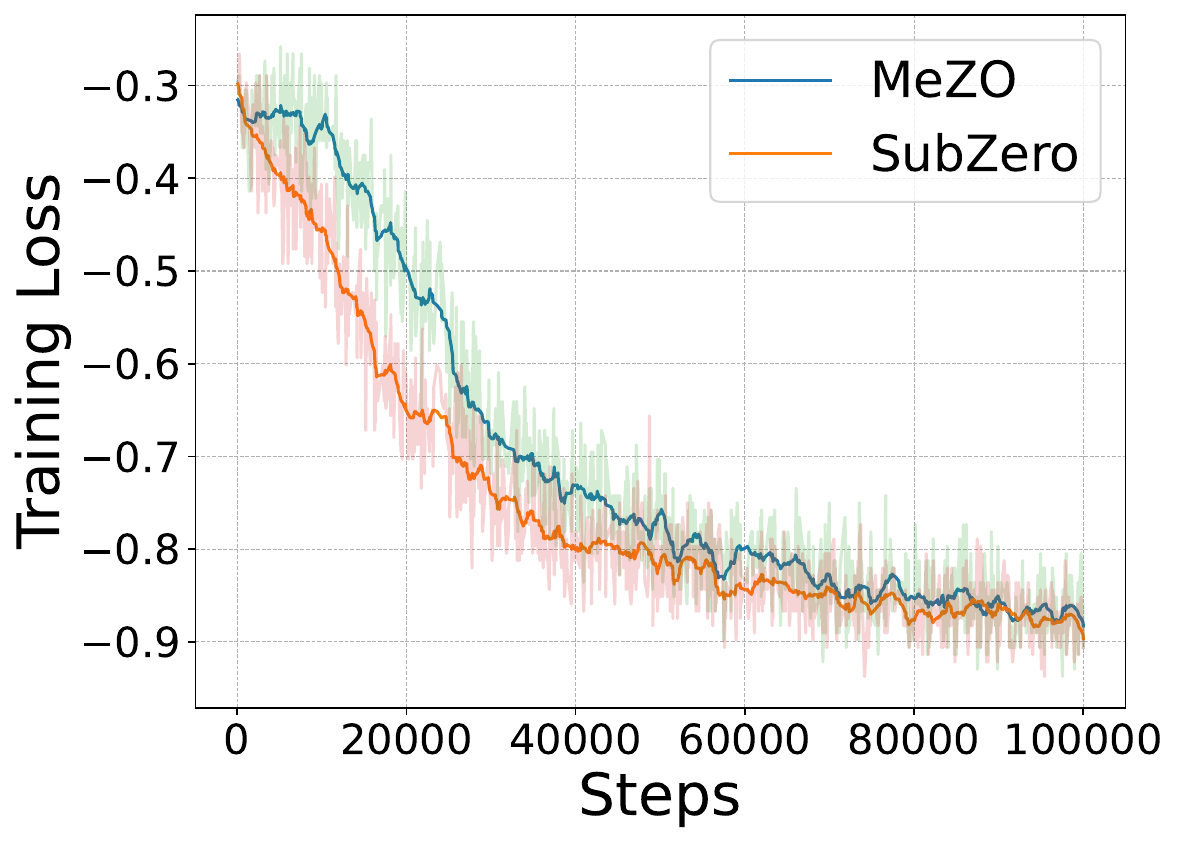}
		\caption{MNLI: RoBERTa, FT, Non-diff}
	\end{subfigure}	

	\vspace{-0.8em}
	\caption{  
		Training loss curves of MeZO and SubZero.}
	\label{fig:train_loss}
	\vspace{-1.6em}
\end{figure*}

\begin{table*}[t]
	\begin{small}
	\renewcommand{\arraystretch}{0.9} 
	\begin{minipage}[t]{0.7\textwidth} 
	\caption{Performance of fine-tuning LLaMA2-7B and Mistral-7B on CB, and OPT-1.3B on SST-2.}
	\vspace{-1em}
	\label{tab:zo_peft}
	\setlength{\tabcolsep}{1.5pt} 
	\centering
	\begin{tabular}{l|cccc|cccc|cccc}
	\toprule
	& \multicolumn{4}{c|}{LLaMA2-7B} & \multicolumn{4}{c|}{ Mistral-7B} & \multicolumn{4}{c}{OPT-1.3B} \\
	& \bf{FT} & \bf{LoRA} & \bf{Prefix} & \bf{Prompt} & 
	\bf{FT} & \bf{LoRA} & \bf{Prefix} & \bf{Prompt} & \bf{FT} & \bf{LoRA} & \bf{Prefix} & \bf{Prompt} \\
	\midrule
	SGD &69.6& 75.0 & 69.6& 69.6&73.2 & 75.0& 69.6& 62.5 & 93.2 & 93.0 & 93.1 & 90.7\\
	\midrule
	MeZO & 64.3& 73.2& 69.6 & 60.7 & 62.5 &69.6 & 58.3&57.1 &92.3 & 92.8& 91.6 & 85.9\\
	SubZero & \bf{71.4} & \bf{75.0} &  \bf{76.8} & \bf{66.1} & \bf{64.3} & \bf{73.2} & \bf{64.3 }& \bf{62.5} &\textbf{93.4}& \textbf{92.9}& \textbf{92.2 } & \textbf{89.1}\\
	\bottomrule
	\end{tabular}
	\end{minipage}
	\hfill
	\begin{minipage}[t]{0.28\textwidth}
	\centering
	\caption{Orthogonal projection matrix.}
	\vspace{-1em}
	\label{tab:abla_projtype}
	\setlength{\tabcolsep}{1.5pt} 
	\renewcommand{\arraystretch}{0.9} 
	\begin{tabular}{l|cc}
	\toprule
	Dataset  & Ortho. & Accuracy \\
	\midrule
	\multirow{ 2}{*}{RTE} &  \ding{55} & 67.5 \\
	& \ding{51} & \textbf{74.0} \\
	\midrule
	\multirow{ 2}{*}{WSC} & \ding{55} & 59.6 \\
	& \ding{51} & \textbf{65.1} \\
	\bottomrule
	\end{tabular}
	\end{minipage}
	\end{small}
	\vspace{-1em}
	\end{table*}

\begin{table}[t]	
	\vspace{-0.6em}
	\caption{Fine-tuning performance comparison between  SubZero and MeZO on RoBERTa-large and OPT-13B with non-differentiable objectives. }
	\vspace{-2em}
	\label{tab:non-diff performance}
	\begin{small}
	\begin{center}
		\setlength{\tabcolsep}{2.2pt}
		\resizebox{1\linewidth}{!}{
			\begin{tabular}{l|cccc|c}
				\toprule
				Model & \multicolumn{4}{c|}{RoBERTa-large} & OPT-13B \\
				Task & 
				\bf{SST-2} & 
				\bf{SST-5} & 
				\bf{SNLI}&
				\bf{MNLI} &
				\bf{SQuAD} \\
				\midrule			
				Zero-shot & 79.0 &35.5& 50.2 &48.8 & 46.2 \\
				Cross entropy (Adam)& 93.9 &55.9& 88.7& 83.8& 84.2 \\
				Cross entropy (MeZO) &92.9 &53.2 & 83.0& 77.0&84.2\\
				Cross entropy (SubZero)  & {93.4} & {54.0}& {84.7}& {77.4} & {84.5}  \\
				\midrule
				Accuracy/F1 (MeZO) & 92.4 &46.5 &81.9 &73.9 & 80.2 \\
				Accuracy/F1 (SubZero) & \textbf{92.7} & \textbf{47.1} & \textbf{83.0}& \textbf{74.8} & \textbf{81.1}\\
				\bottomrule
			\end{tabular}
		}
	\end{center}
	\end{small}
	\vspace{-2.5em}
\end{table}

In this section, we present comprehensive experiments to evaluate the effectiveness of SubZero. We conduct our experiments using medium-sized masked LLMs (RoBERTa-large~\citep{liu2019roberta}) and large-scale autoregressive LLMs (OPT-1.3B and 13B~\citep{zhang2022opt}, LLaMA2-7B~\citep{touvron2023llama}, and Mistral-7B~\citep{jiang2023mistral}). Our exploration covers full-parameter tuning (FT)~\citep{aghajanyan2021intrinsic} and three PEFT schemes: LoRA~\citep{hu2021lora}, prefix tuning~\citep{li2021prefix}, and prompt tuning~\citep{lester2021power}. For comparison, we include leading ZO methods, such as MeZO~\citep{malladi2023fine}, ZO-AdaMU~\citep{jiang2024zo}, S-MeZO~\citep{liu2024sparse},  HiZOO~\citep{zhao2024second}, and LOZO~\citep{chen2024enhancing} alongside inference-only memory-efficient baselines like zero-shot, in-context learning (ICL)~\citep{brown2020language}, and linear probing (LP)~\citep{kumar2022fine}. As the first and most popular ZO optimizer for LLM fine-tuning, MeZO is considered our primary competitor. We also use the FO optimizer SGD as a benchmark. Since appropriate prompts are critical for ZO optimization~\citep{malladi2023fine,zhang2024revisiting}, all experiments incorporate prompt templates. 
Since a larger batch size reduces the variance of the estimated gradients, all compared methods use a batch size of 16 unless otherwise specified. All experimental settings are detailed in Appendix~\ref{app:details}-\ref{sec:prompt}. 

\subsection{Results under Different Experimental Settings}

Following the settings in MeZO~\citep{malladi2023fine}, we evaluated SubZero using OPT-13B on the SuperGLUE benchmark~\citep{wang2019superglue}, which covers a diverse range of tasks, including classification, multiple-choice, and generation, as outlined in Table~\ref{tab:opt13b_performance}.  The ZO methods were applied to both full-parameter tuning (FT) and LoRA fine-tuning schemes. The comparisons with vanilla LoRA and SGD with gradient accumulation  are provided in Appendix~\ref{sec:add_results}.

Table~\ref{tab:opt13b_performance} presents the key findings, highlighting the best-performing ZO method in bold. The results show that ZO techniques significantly outperform baseline approaches like zero-shot, in-context learning, and linear probing, underscoring their ability to enhance a pre-trained model’s performance on downstream tasks. From Table~\ref{tab:opt13b_performance}, one can also observer that MeZO, the first ZO optimizer for LLM fine-tuning, is highly competitive after carefully tuning its hyperparameters. Only ZO-AdaMU and LOZO, apart from SubZero, outperform MeZO in FT scheme.  SubZero consistently surpasses MeZO across all tasks and fine-tuning schemes. 
In FT scheme, S-MeZO showed competitive performance on several classification tasks. However, its performance on ReCoRD remained unsatisfactory even after hyperparameter tuning.  
Excluding ReCoRD, SubZero still outperforms S-MeZO with 2.05\%  vs. 1.20\% in FT scheme and 1.74\% vs. -4.90\% in LoRA scheme.   
 
We further extended our evaluation of SubZero using OPT-1.3B, LLaMA2-7B, and Mistral-7B in FT and three PEFT schemes: LoRA, prefix tuning, and prompt tuning. As shown in Table~\ref{tab:zo_peft}, SubZero outperformed MeZO across all models and fine-tuning schemes. Notably, while MeZO struggled in the prompt tuning scheme, SubZero excelled, achieving performance levels that closely matched those of the SGD optimizer.

We also present several training loss curves in \mbox{Fig.~\ref{fig:train_loss} (a)-(c)}, demonstrating that SubZero generally converges faster and achieves lower losses compared to MeZO. 
 
\subsection{Results on  Non-Differentiable Objectives}
Following MeZO~\citep{malladi2023fine}, we apply SubZero to fine-tune RoBERTa-large and OPT-13B using full-parameter fine-tuning with two non-differentiable objectives—optimizing accuracy in classification tasks and F1 in the SQuAD task. As baselines, we also report results using the differentiable cross-entropy objective with Adam, MeZO, and SubZero. As shown in Table~\ref{tab:non-diff performance}, SubZero consistently outperforms MeZO with the two non-differentiable objectives. The training loss curves in \mbox{Fig.~\ref{fig:train_loss} (d)} also demonstrate the advantage of SubZero over MeZO.


\subsection{Memory Usage and Wall-Clock Time Analysis}\label{sec:memory_and_time}
We compare the memory consumption and wall-clock time of ZO methods (MeZO and SubZero), SGD, and inference-only approaches (zero-shot and in-context learning (ICL)) using OPT-13B (see Table~\ref{tab:memoryusage} in Appendix~\ref{sec:add_results}). Since inference-only methods do not involve fine-tuning, they have zero wall-clock time and their memory usage reflects only the inference load. For fine-tuning, all methods were run for 20K steps. The ZO methods, including SubZero, achieved over a 1.8× reduction in memory usage compared to SGD. Notably, SubZero’s memory footprint closely aligns with MeZO’s, while offering improved performance. We use per-layer weight updates for MeZO and SubZero (see Appendix~\ref{app:details}), resulting in nearly identical memory usage for FT and LoRA schemes when one decimal place is reserved.

Although SubZero introduces additional computational overhead for generating projection matrices via QR decomposition, the observed extra time overhead remains below 9\% across all OPT model sizes, with a mean value of 4.79\% of the total wall-clock time (see Table~\ref{tab:opt_QR_comparison} in Appendix~\ref{sec:add_results}). 

Note that due to differences in how steps are defined between ZO methods and SGD, direct wall-clock time comparisons between the two are not entirely meaningful.

 \subsection{Ablation Study}\label{sec:ablation}
  
 We conducted a thorough investigation of the effectiveness of our techniques. Table~\ref{tab:abla_projtype} shows that using a column-orthogonal projection matrix significantly outperforms a Gaussian random projection matrix, primarily due to the low-rank structure of the perturbation matrices (see experimental settings in Appendix~\ref{sec: hyperparam}). This low-rank perturbation is key to improving the quality of gradient estimation. 
 
 Next, Table~\ref{tab:abla-rank-T} explores the effects of subspace rank \(r\) and update frequency \(F\) in Algorithm~\ref{alg:ss_mezo}. The results demonstrate that SubZero is robust to variations in the subspace rank. However, performance drops sharply when the update frequency is too low, as the optimization becomes constrained to a single subspace for too long, limiting its adaptability. 
 
\begin{table}[ht] 
	\begin{small}
    \vspace{-0.6em}
    \begin{minipage}{0.2\textwidth}
        \centering
		\begin{small}
        \caption{Subspace change frequency $F$ and rank $r$.}
        \vspace{-1em}
        \label{tab:abla-rank-T}
        \setlength{\tabcolsep}{2.5pt}
        \renewcommand{\arraystretch}{0.92} 
        \resizebox{0.87\textwidth}{!}{
            \begin{tabular}{l|ccc}
                \toprule
                $F$ \textbackslash~$r$ & 32 & 64 & 128 \\
                \midrule
                500 & \textbf{72.6} & 70.0  & 72.2\\
                1000& 73.6 & 71.8& \textbf{74.0}\\
                2000 & 72.2 &\textbf{73.3} &72.2\\
                \midrule
                20000 & 70.4& \textbf{71.1}&68.6 \\
                \bottomrule
            \end{tabular}
        }
		\end{small}
    \end{minipage}
    \hspace{0.5em}
    \begin{minipage}{0.25\textwidth}
        \centering
        \caption{Reshaping strategy for non-square matrices on SST-2 with OPT-1.3B in PEFT schemes.}
        \vspace{-0.75em}
        \label{tab:abla_reshape}
        \setlength{\tabcolsep}{1pt}
        \renewcommand{\arraystretch}{0.9} 
        \resizebox{\textwidth}{!}{
            \begin{tabular}{l|ccc}
                \toprule
                Method & LoRA & Prefix & Prompt \\
                \midrule
                MeZO & 92.8 & 91.6 & 85.9 \\
                SubZero(w/o) & 92.1& 89.4& 74.2 \\
                SubZero(w/) & \textbf{92.9} & \textbf{92.2} & \textbf{89.1} \\
                \bottomrule
            \end{tabular}
        }
    \end{minipage}
    \vspace{-0.65em}
\end{small}
\end{table}

 Finally, Table~\ref{tab:abla_reshape} underscores the critical role of the reshaping strategy for handling highly non-square perturbation matrices,  essential for ensuring effective perturbations in different layers of the model. Together, these results highlight the improvements brought by our design choices, particularly in terms of projection and reshaping strategies, and their impact on SubZero's robustness and performance.  
 
 Due to limited space, the ablation studies of SubZero on random seed, batch size, and combination with Adam are given in Appendix~\ref{sec:add_results}.   
 
\section{Conclusion} 
We have demonstrated that SubZero effectively fine-tunes large LLMs across various tasks and schemes with a memory cost comparable to that of inference. Extra experiments indicate that SubZero can optimize non-differentiable objectives. Our theory explains how SubZero reduces the variance of gradient estimates and hence accelerates convergence. 
 
\setlength{\parskip}{0.4\baselineskip}
\noindent\textbf{Limitation.}  In addition to the representative first-order and primitive zero-order optimizers, we have yet to investigate the combinations of SubZero with other first-order and zero-order optimizers to evaluate the implications on convergence speed. While SubZero is compatible with memory-efficient techniques like parameter quantization~\citep{li2024memory}, we have not thoroughly explored the practical effects of these combinations. A theoretical analysis of the reshaping strategy is valuable but left for future work.

\clearpage

\section*{Acknowledgments}
This work is supported by the NSF of China (grant nos. 62131003 and 62102034), Beijing Municipal Science and Technology Project (Z241100001324011), the Ministry of Education, Singapore, under its AcRF Tier 2 Funding (Proposal ID: T2EP20224-0048). Any opinions, findings and conclusions or recommendations expressed in this material are those of the author(s) and do not reflect the views of the Ministry of Education, Singapore.

{
    \small
    \bibliographystyle{ieeenat_fullname}
    \bibliography{main}
}

\clearpage
\onecolumn  
\section{Appendix}

\subsection{Additional Results}\label{sec:add_results}

\noindent\textbf{Empirical Evidence of BP Gradients Converging to Subspaces}

Previous studies have shown that BP gradients during fine-tuning of LLMs rapidly converge to low-dimensional subspaces. Building on this, our empirical investigation using the OPT-1.3B architecture on the SST-2 dataset uncovers a persistent low-rank structure in gradient matrices throughout the optimization process. As shown in Fig.~\ref{fig:appendix_lowrank_grad}, singular value decomposition (SVD) of gradient matrices across layers consistently reveals pronounced spectral decay, with only a small subset of singular values dominating the gradient spectrum during LLM fine-tuning with SGD.

\begin{figure*}[ht]	
	\centering	
	\begin{subfigure}{0.39\textwidth}
		\includegraphics[width=\textwidth]{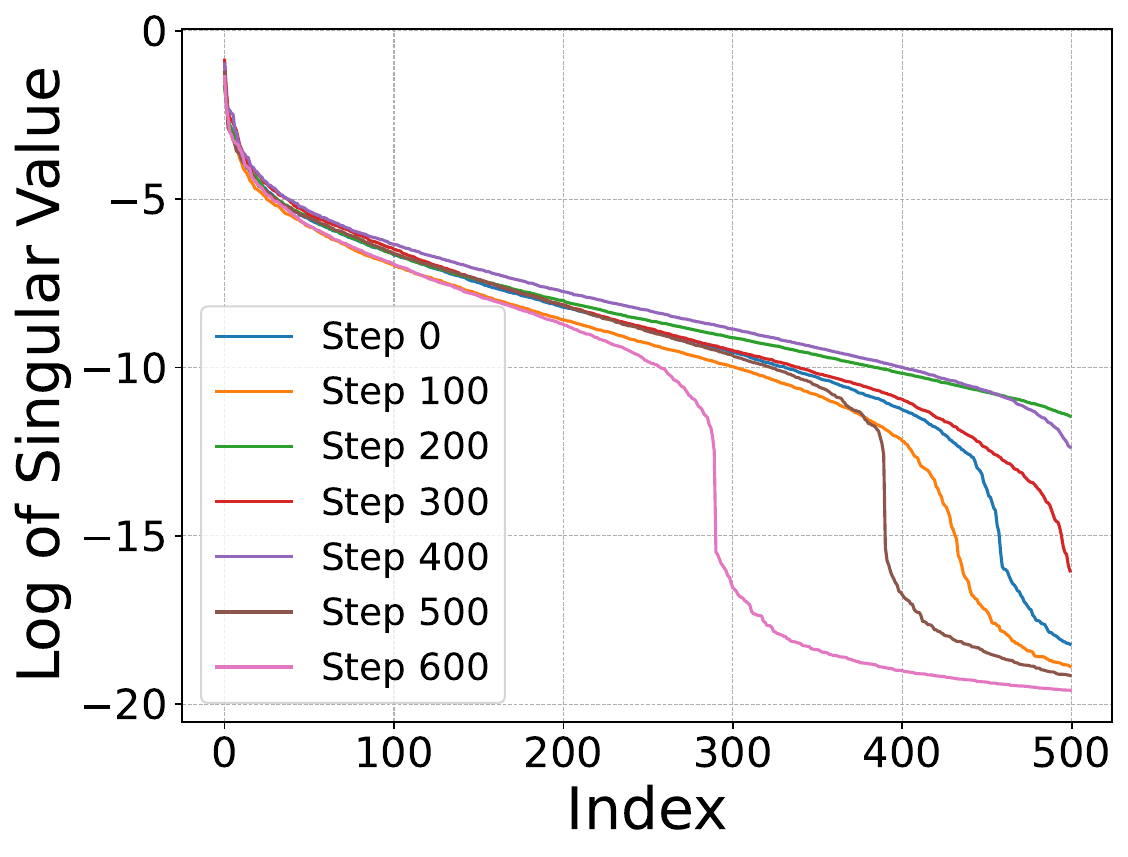}
		\caption{Singular values of \texttt{k\_proj} matrix for Layer 22.}
	\end{subfigure}
	\begin{subfigure}{0.39\textwidth}
		\includegraphics[width=\textwidth]{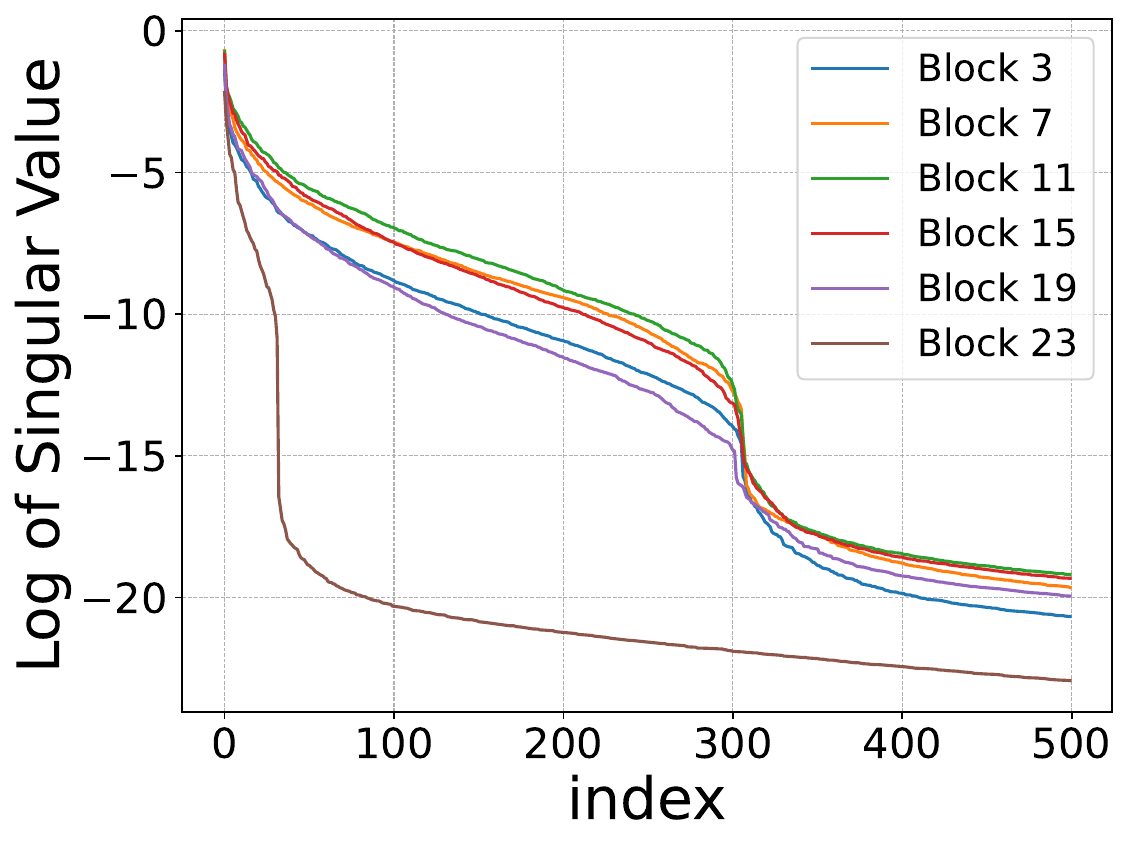}
		\caption{Singular value of \texttt{q\_proj} matrix graident over step 100.}	
	\end{subfigure}
	\vspace{-0.8em}
	\caption{Low-rank structures in gradient matrices during LLM fine-tuning. (a) illustrates the temporal evolution of singular values for key projections within a single layer, showing consistent spectral decay across training steps. (b) presents the cross-layer comparison at step 100, revealing layer-invariant low-rank patterns in query projection gradients.} 
	\label{fig:appendix_lowrank_grad}
\end{figure*}

\noindent\textbf{More Comparisons }

In the main manuscript, we use the identical batch size for FO and ZO optimizers. Here, we adjust SGD with gradient accumulation to match the memory usage of ZO optimizers, and then compare their convergence speed and performance. The experimental settings are the same as those in Fig.~\ref{fig:motivation}, and the experimental results are shown in Fig.~\ref{fig:appendix_ga_exper}. With similar  memory usage, SubZero attains a convergence rate nearly on par with SGD, surpasses MeZO, and achieves test accuracy comparable to that of SGD. 

\begin{figure*}[ht]	
	\centering	
	\begin{subfigure}{0.31\textwidth}
		\includegraphics[width=\textwidth]{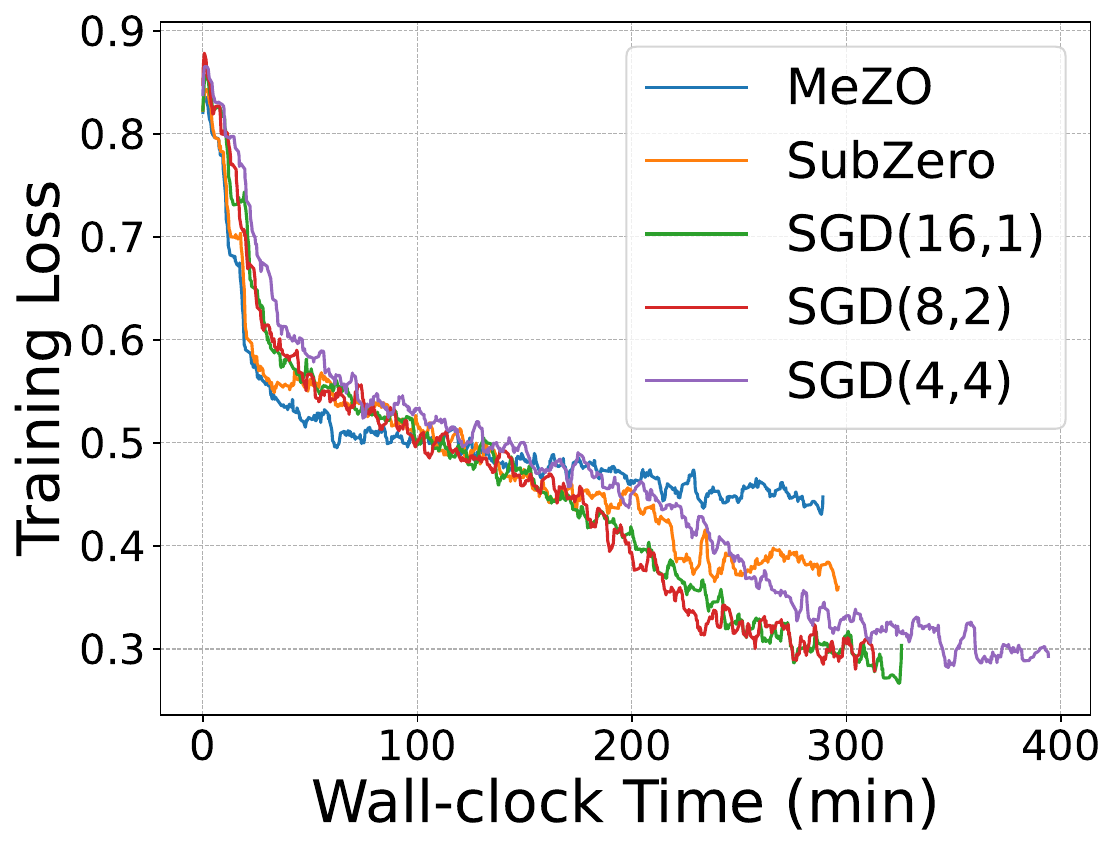}
		\caption{Training Loss}
	\end{subfigure}
	\hfill
	\begin{subfigure}{0.32\textwidth}
		\includegraphics[width=\textwidth]{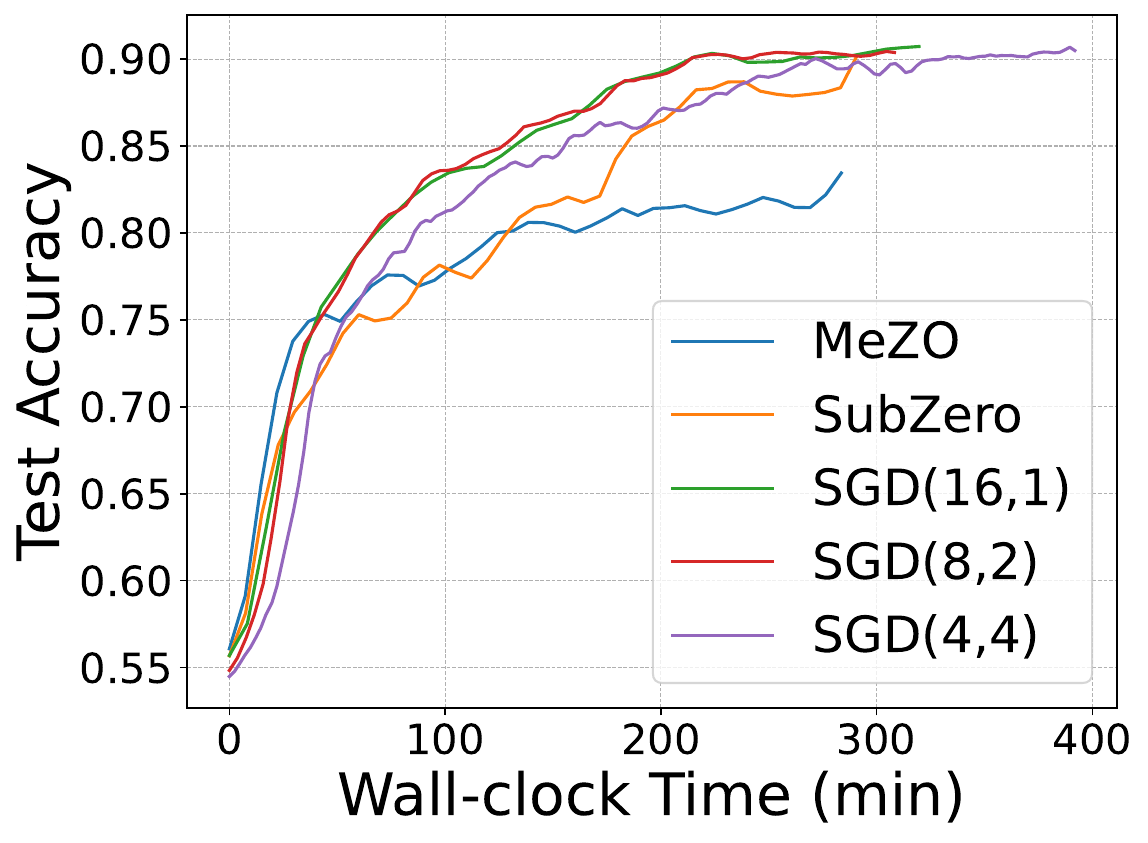}
		\caption{Test Accuracy}		
	\end{subfigure}
	\hfill
	\begin{subfigure}{0.31\textwidth}
		\includegraphics[width=\textwidth]{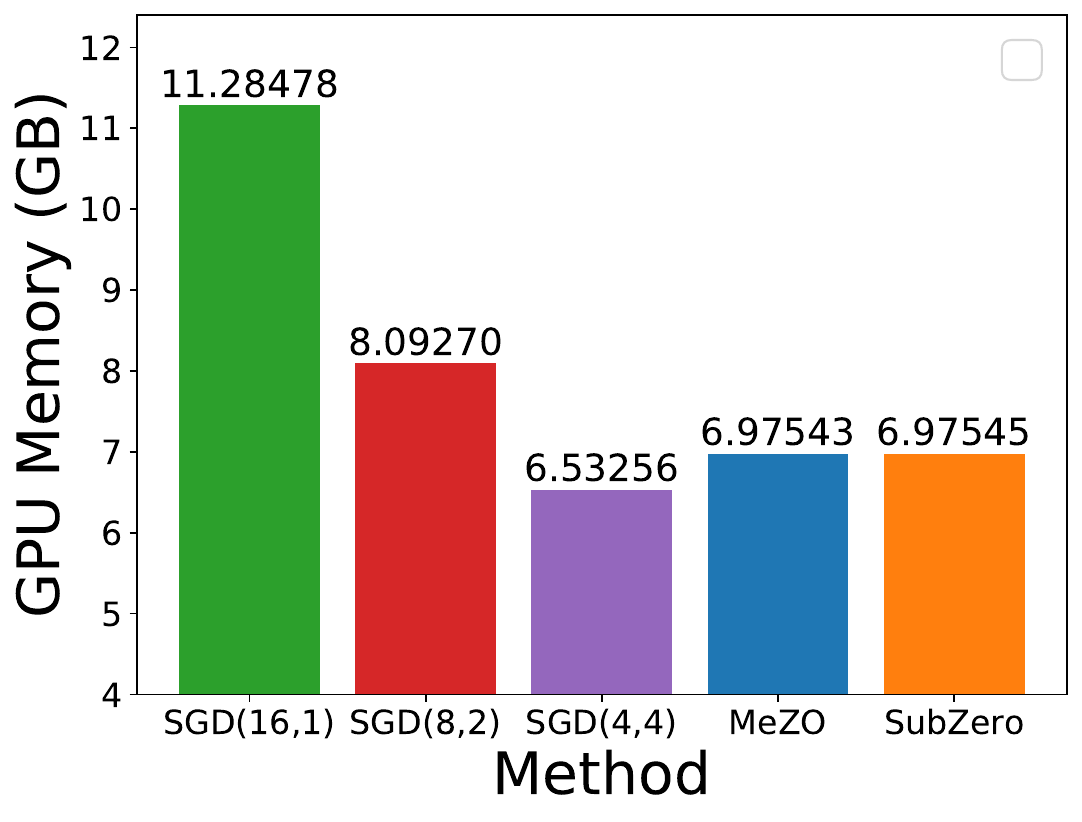}
		\caption{Memory Cost}		
	\end{subfigure}	
	\vspace{-0.8em}
	\caption{ Visualization of training loss, test accuracy, and peak total GPU memory usage with OPT-1.3B on SST-2 in prompt tuning scheme. SGD(BS, GA) refers to SGD with a batch size of BS and GA times of gradient accumulation. All ZO methods utilize a batch size of 16, while SGD(BS, GA) applies gradient accumulation to ensure its memory usage aligns with that of the ZO optimizers. All methods are executed for 20K steps.}
	\label{fig:appendix_ga_exper}
\end{figure*}

The vanilla LoRA is fine-tuned by Adam. We compare SubZero with SGD in the FT and LoRA schemes with vanilla LoRA using the pretrained OPT-1.3B model  on SST-2. For Adam and SubZero with SGD, we apply the constant learning rate schedule. The results are given in Table~\ref{tab:vanilla_LoRA}. We can see that SubZero with SGD in the FT scheme outperforms vanilla LoRA in both test accuracy and memory usage. SubZero with SGD in the LoRA scheme also achieves comparable test accuracy while maintaining minimal memory usage.

\begin{table}[ht]	
	\caption{Comparison with vanilla LoRA using the pretrained OPT-1.3B model  on SST-2.}
	\vspace{-1.5em}
	\begin{small}
	\label{tab:vanilla_LoRA}
	\begin{center}
		\begin{tabular}{l|ccc|c}
			\toprule		
			Method & Test Accuracy(\%) & Total Memory (GB) \\
			\midrule
			LoRA (Adam) & 93.2 & 10.75 \\
			SubZero (FT) &\textbf{93.4} & ~6.88 \\
			SubZero (LoRA) &92.9 & ~\textbf{6.80} \\			
			\bottomrule
		\end{tabular}
	\end{center}
\end{small}
\end{table}

As described in Sec.~\ref{sec:memory_and_time}, we compare the memory consumption and wall-clock time of ZO methods (MeZO and SubZero), SGD, and inference-only approaches (zero-shot and in-context learning (ICL)) using OPT-13B (see Table~\ref{tab:memoryusage}). Since inference-only methods do not involve fine-tuning, they have zero wall-clock time and their memory usage reflects only the inference load. For fine-tuning, all methods were run for 20K steps. The ZO methods, including SubZero, achieved over a 1.8× reduction in memory usage compared to SGD. Notably, SubZero’s memory footprint closely aligns with MeZO’s, while offering improved performance. We use per-layer weight updates for MeZO and SubZero (see Appendix~\ref{app:details}), resulting in nearly identical memory usage for FT and LoRA schemes when one decimal place is reserved.

Although SubZero introduces computational and memory overhead due to QR decomposition when generating projection matrices, our empirical analysis reveals strictly bounded resource costs across all tested OPT model scales (see Table~\ref{tab:opt_QR_comparison}). Specifically, the additional time overhead remains below 8.5\% (peaking at 8.36\% for the 6.7B model), while the memory overhead stays under 1.8\%, even under bfloat16 precision. This indicates that the computational cost of QR decomposition becomes asymptotically negligible as model complexity increases, thereby establishing SubZero's practical scalability.

\begin{table}[ht]
	\centering
	\caption{Memory usage (GB) and wall-clock time (minutes) of fine-tuning OPT-13B, with SGD's batch size being 8 for SQuAD and 16 for other tasks. } 
\label{tab:memoryusage}
\setlength{\tabcolsep}{7pt}
\begin{small}
\begin{threeparttable}
	\begin{tabular}{l|cc|cc|cc}
		\toprule
		Task & \multicolumn{2}{c|}{SST-2} & \multicolumn{2}{c|}{WIC} & \multicolumn{2}{c}{SQuAD} \\
		Method & \bf{Mem.} & \bf{Time} & \bf{Mem.} & \bf{Time} & \bf{Mem.} & \bf{Time} \\
		\midrule			
		Zero-shot/ICL & 24.2 & 0 & 24.8 & 0& 27.2 & 0 \\
		SGD(FT) & 48.9 & 190.3 & 48.9 & 257.3 & 122.7 & 623.7 \\
		\midrule
		MeZO(FT) & 26.1 & 324.9& 26.6 &370.5 & 37.4 &  670.2 \\
		SubZero(FT) & 26.5 & 337.3 & 27.1& 385.3 & 37.8 & 690.5\\
		\midrule
		MeZO(LoRA) & 26.1 & 123.9 & 26.6 & 171.6 & 37.4 & 476.7\\
		SubZero(LoRA) & 26.1& 130.3 & 26.6 & 179.7 & 37.4 & 486.5\\
		\bottomrule
	\end{tabular}
\end{threeparttable}
\end{small}
\end{table} 

\begin{table}[ht]
	\caption{Peak memory usage (GB) and wall-clock time (seconds) for fine-tuning OPT models on the SST-2 dataset. All models fine-tuned on SST-2 for 20K steps. Precision strategy: FP32 for models $\le$ 6.7B, BF16 for $\geq$ 6.7B.}
	\begin{small}
	\vspace{-1em}
	\begin{center}
	\setlength{\tabcolsep}{2pt}
		\begin{tabular}{@{}cc|ccc|ccc@{}}
	\toprule
	OPT series & hidden size & MeZO Mem. & SubZero Mem. & \begin{tabular}[c]{@{}c@{}}Mem.\\ Overhead (\%)\end{tabular} & MeZO Time & SubZero Time & \begin{tabular}[c]{@{}c@{}}Time\\ Overhead (\%)\end{tabular} \\
	\midrule
	1.3B & 2048 & 6.80 & 6.88 & \textbf{+1.18\%} & 11214.95 & 11683.39 & \textbf{+4.18\%} \\
	2.7B & 2560 & 12.40 & 12.60 & \textbf{+1.61\%} & 21578.56 & 22335.05 & \textbf{+3.51\%} \\
	6.7B & 4096 & 13.98 & 14.20 & \textbf{+1.57\%} & ~9832.69 & 10654.59 & \textbf{+8.36\%} \\
	13B  & 5120 & 26.08 & 26.53 & \textbf{+1.73\%} & 18667.50 & 19245.10 & \textbf{+3.09\%} \\
	\bottomrule
		\end{tabular}
	\end{center}
	\end{small}
	\label{tab:opt_QR_comparison}
\end{table}

\noindent\textbf{Integration with Other ZO Optimizers}

SubZero is orthogonal to many ZO methods and can be combined with them to further boost performance, such as the momentum mechanism in ZO-AdaMU~\cite{jiang2024zo}, the sparsity pruning in S-MeZO~\cite{liu2024sparse}, and the second-order information in HiZOO~\cite{zhao2024second}. Here, we report the experimental results of integrating SubZero with ZO-AdaMU, as shown in Table~\ref{tab:subzero_momentum_mem_acc} and Fig.~\ref{fig:subzero_momentum_loss_curve}. This integration not only achieves faster convergence but also significantly reduces the memory overhead associated with ZO-AdaMU’s momentum mechanism.

\begin{figure}[ht]  
    \centering  
    \begin{minipage}{0.45\textwidth}  
        \centering  
		\Large
		\renewcommand{\arraystretch}{1.3} 
		\resizebox{0.85\textwidth}{!}{
       \begin{tabular}{l|cc}
		  \toprule
          
		      Method & Memory (GB) & Accuracy (\%) \\
              \midrule
		    MeZO &  \textbf{26.62} & 60.0 \\
		    SubZero & 27.06 & 60.8  \\ 
        ZO-AdaMU & 50.77 & 60.7 \\
        SubZO-AdaMU & 27.07 & \textbf{61.1} \\ 
        \bottomrule
	\end{tabular}}
        \captionof{table}{Peak memory usage (GB) and test accuracy (\%) for fine-tuning OPT-13B models on the WIC dataset.}  
        \label{tab:subzero_momentum_mem_acc}  
    \end{minipage}  
    \begin{minipage}{0.40\textwidth}  
        \centering  
        \includegraphics[width=0.85\linewidth]{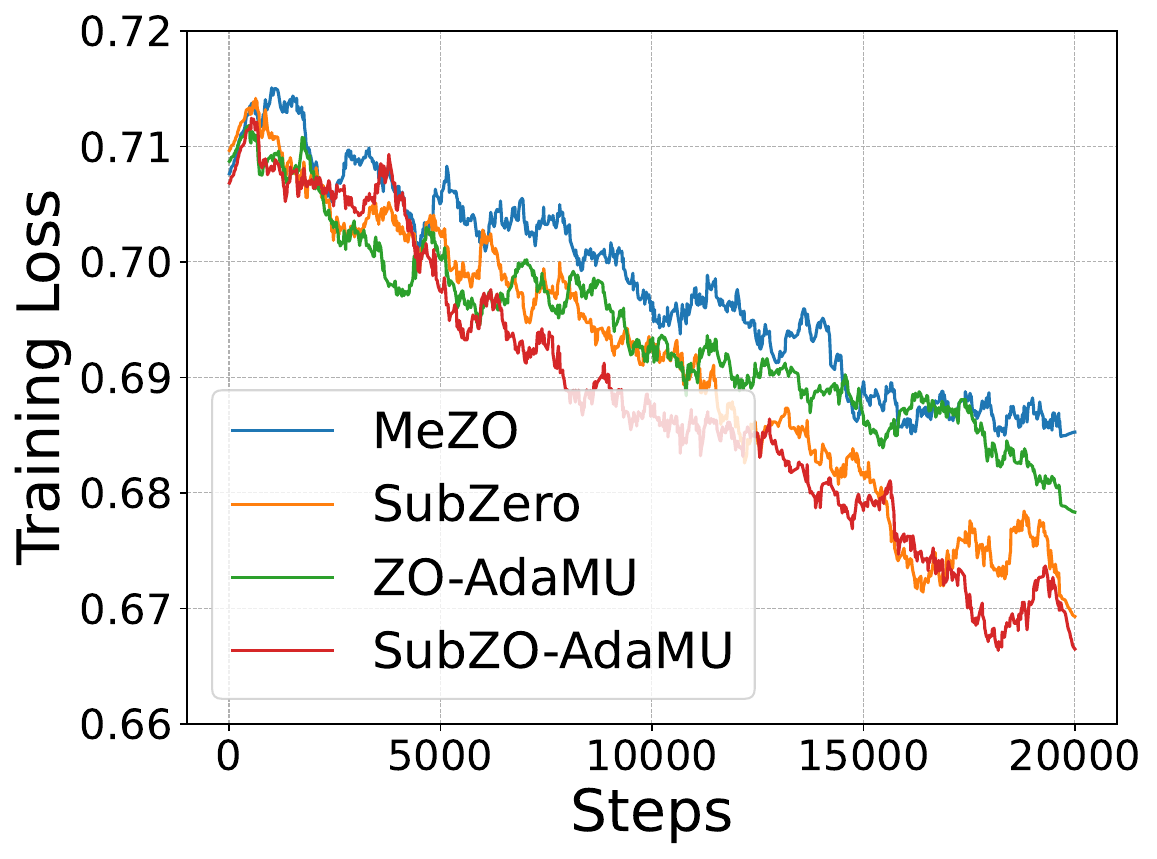}  
        \vspace{-1em}
        \captionof{figure}{Training loss curves of Table~\ref{tab:subzero_momentum_mem_acc}.}  
        \label{fig:subzero_momentum_loss_curve}  
    \end{minipage}  
\end{figure}

ZO-AdaMU accelerates convergence by introducing stochastic momentum, a mechanism that is orthogonal to SubZero. In all experiments, we use a shared learning rate and perturbation scale. SubZero-specific hyperparameters, such as rank and update interval, are set as detailed in Table~\ref{tab:opt13b_performance}, while all momentum-related settings remain consistent with the original ZO-AdaMU paper~\cite{jiang2024zo}. Notably, vanilla ZO-AdaMU requires storing the full momentum history, leading to substantial memory overhead. In contrast, when combined with SubZero, only low-dimensional momentum $\boldsymbol{M}_i^{t} \in \mathbb{R}^{r \times r}$ needs to be maintained, i.e.,
\begin{align}
\boldsymbol{M}_i^{t} &= \beta_i^{t} \boldsymbol{M}_i^{t - 1} + (1 - \beta_i^{t}) \boldsymbol{Z}_i^{t}, \label{eq:subzero1} \\
\boldsymbol{W}_i^{t} &= \boldsymbol{W}_i^{t - 1} - \eta_i^{t} \boldsymbol{U}_i \boldsymbol{M}_i^{t} \boldsymbol{V}_i^T. \label{eq:subzero_momentum}
\end{align}
When the subspaces are updated at time steps $t$ such that $t\mod F  \equiv 0$ (as specified in Algorithm~\ref{alg:ss_mezo}), the old momentum needs to be projected from the old subspaces $\boldsymbol{U}^{t-1}$ and $\boldsymbol{V}^{t-1}$ onto the new subspaces $\boldsymbol{U}^{t}$ and $\boldsymbol{V}^{t}$. This projection ensures the optimizer's continuity and maintains the effectiveness of the momentum mechanism. The projection is formulated as the following optimization problem:
\begin{align}
\min_{M}\|\boldsymbol{U}^{t-1}\boldsymbol{M}^{t-1}(\boldsymbol{V}^{t-1})^{T}-\boldsymbol{U}^{t}\boldsymbol{M}(\boldsymbol{V}^{t})^{T} \| . \label{eq:subzero_momentum2}
\end{align}

\noindent\textbf{More Ablation Studies}

We investigate the effects of random seed, batch size, and combination with Adam for SubZero. We also provide more results on the reshaping strategy. 

We first fine-tune the OPT-1.3B model on the SST-2 dataset in prompt tuning scheme with three random seeds. We present the results in Table~\ref{tab:randomseed}, and hyperparameters are presented in Table~\ref{tab:hyper_grid2}. For various random seeds, the variance of MeZO is quite large, whereas the variance of SubZero is small, and its average performance is superior.

Then we examine the impact of batch size for ZO optimizers using the RoBERTa-large model on SST-2 in full-parameter tuning scheme. The results are shown in Table~\ref{tab:batchsize}. The training epochs are 100K in Table~\ref{tab:non-diff performance}, while they are 20K in Table~\ref{tab:batchsize}. The remaining hyperparameters are consistent with Table~\ref{tab:non-diff performance}, as detailed in Appendix~\ref{sec: hyperparam}. For ZO optimizers, a large batch size always gets better performance. Across various batch sizes, SubZero demonstrates better fine-tuning performance compared to MeZO. 

\begin{table}[ht]	
	\caption{The impact of random seed with the pretrained OPT-1.3B model  on SST-2 in prompt tuning scheme.}
	\vspace{-1.5em}
	\label{tab:randomseed}
	\begin{small}
	\begin{center}
		\begin{tabular}{l|ccc|c}
			\toprule
			Seed & 42 & 0 & 1234 & AVG. \\
			\midrule
			MeZO & 85.9 & 83.3 & 80.7 & 83.3 \\
			SubZero &\textbf{89.1} & \textbf{89.4} & \textbf{89.2} & \textbf{89.2}\\			
			\bottomrule
		\end{tabular}
	\end{center}
	\end{small}
\end{table}

\begin{table}[ht]	
	\caption{The impact of batch size with the pretrained RoBERTa-large model in full-parameter tuning scheme.}
	\vspace{-1.5em}
	\label{tab:batchsize}
	\begin{small}
	\begin{center}
		\begin{tabular}{c|c|cccc|c}
			\toprule
			{Batch Size} & {Method} & \textbf{SST-2} & \textbf{SST-5} & \textbf{SNLI} & \textbf{MNLI} & AVG.\\
			\midrule
			\multirow{2}{*}{16} & MeZO & 91.7 & 44.7 & 77.3 & \textbf{53.0} & 66.7\\
			& SubZero & \textbf{91.9} & \textbf{45.9} & \textbf{77.5} & 52.8 & \textbf{67.0}\\
			\midrule
			\multirow{2}{*}{32} & MeZO & 92.9 & 45.4 & 78.3 & 53.2 & 67.5\\
			& SubZero & \textbf{93.0} & \textbf{45.5} & \textbf{79.6} & \textbf{54.0} & \textbf{68.0}\\
			\bottomrule
		\end{tabular}
	\end{center}
\end{small}
\end{table}

Next, we assess the impact of the Adam optimizer. We fine-tune the OPT-1.3B model on the SST-2 dataset, and the experimental results are displayed in Table~\ref{tab:adam}. For ZO optimizers with Adam, we perform a grid search on the hyperparameters and find that keeping the learning rate and perturbation scale consistent with those with SGD resulted in good convergence, as detailed in  Table~\ref{tab:hyper_grid2}. We utilize the linear and the constant learning rate schedules for SGD and Adam, respectively. For all ZO optimizers with SGD, we apply the constant learning rate schedule. For all ZO optimizers with Adam, we test the constant and the cosine annealing schedules. We note that SubZero surpasses MeZO when employing the Adam optimizer with both constant and cosine annealing schedules. Also, Adam does not provide an advantage over SGD for ZO optimization, which aligns with the conclusions of previous studies~\citep{zhang2024revisiting, guo2024zeroth}.

\begin{table}[ht]	
	\caption{Comparison of test accuracy (\%) for the pretrained OPT-1.3B model fine-tuned on SST-2 with SGD and Adam.}
	\vspace{-1.5em}
	\label{tab:adam}
	\begin{small}
	\begin{center}
		\begin{tabular}{l|cccc|c}
			\toprule
			Method & 
			\bf{FT} & 
			\bf{LoRA} & 
			\bf{Prefix}&
			\bf{Prompt} & AVG.\\
			\midrule
			SGD &\bf{93.2} &93.0 &\bf{93.1} &90.7 & 92.5\\
			Adam &92.6 &\bf{93.2} &92.9 & \bf{93.3} & \bf{93.0}\\
			\midrule			
			MeZO\_SGD & 92.3 & 92.8 & 91.6 & 85.9 & 90.7\\
			SubZero\_SGD & \bf{93.4} & \bf{92.9} & \bf{92.2} & \bf{89.1} & \textbf{91.9}\\
			\midrule
			MeZO\_Adam(constant) & 92.3 & \bf{93.3} & 90.7 & 84.6 & 90.2\\
			SubZero\_Adam(constant) & \bf{93.2} & 92.4 & \bf{90.9} & \bf{89.3} & \textbf{91.5}\\
			\midrule
			MeZO\_Adam(cosine) & \bf{91.9} & \bf{93.1} & 86.1 & 78.7 & 87.5\\
			SubZero\_Adam(cosine) & 91.7 & 92.0 & \bf{86.6} & \bf{83.4} & \textbf{88.4}\\
			\bottomrule
		\end{tabular}
	\end{center}
\end{small}
\end{table}

Finally, we provide more ablations on the reshaping strategy with OPT-1.3B on Winogrande in the PEFT schemes. The Winogrande dataset~\cite{sakaguchi2021winogrande} is a benchmark for commonsense reasoning and available at~\url{https://winogrande.allenai.org/}. The results are shown in Table~\ref{tab:reshape_winogrande}. We can set that the reshaping strategy clearly enhances performance, aligning with the conclusion presented in Table~\ref{tab:abla_reshape}.

\begin{table}[ht]	
	\caption{Reshaping strategy for non-square matrices with the pretrained OPT-1.3B model fine-tuned on Winogrande in PEFT schemes.}
	\vspace{-1.5em}
	\label{tab:reshape_winogrande}
	\begin{small}
	\begin{center}
		\begin{tabular}{l|ccc|c}
			\toprule
			Method & 
			\bf{LoRA} & 
			\bf{Prefix}&
			\bf{Prompt} & AVG. \\
			\midrule		
			SGD & 58.3	& 56.9	& 58.4 & 57.9\\
			\midrule
			SubZero(w/o) &56.6&	56.6&	56.5	& 56.6 \\
			SubZero(w/) & \textbf{57.8} &	\textbf{57.3}&	\textbf{57.6} & \textbf{57.6}\\	 
			\bottomrule
		\end{tabular}
	\end{center}
\end{small}
\end{table}

\subsection{Implementation Details} \label{app:details}

We use one A800 GPU with the PyTorch 2.1.0+CUDA 11.8 framework for ZO methods and, if needed,  two A800 GPUs for SGD.

The gradient estimation in SubZero is applicable to parameter matrices, while LLMs mainly consist of dense layers. For other trainable parameters, such as biases and layer normalization parameters, we recommend using the gradient estimation in MeZO~\citep{malladi2023fine}, as these layers contain fewer parameters.  

We introduce two useful strategies to implement our SubZero efficiently in memory. 

\textbf{In-place Operation.} As indicated in Eqn.~\eqref{eq:cal_rho}, directly computing the loss  difference $\rho$ requires twice the memory of inference, as it must store both the parameter matrix set $\mathcal{W}$ and the perturbation matrix set  $\tilde{\mathcal{Z}}$. To mitigate this, we draw inspiration from MeZO and utilize in-place operations. By employing the random seed trick, we store a random seed to compute $\rho$ (see lines 9-12 in Algorithm~\ref{alg:ss_mezo} and Algorithm~\ref{alg:subspaceperturb}) and regenerate the low-dimensional perturbation matrices $\boldsymbol{Z}_{1},\boldsymbol{Z}_{2},\cdots,\boldsymbol{Z}_{l}$ (see line 15 in Algorithm~\ref{alg:ss_mezo}). Consequently, the memory cost for fine-tuning with SubZero is nearly equivalent to that of inference (see Table~\ref{tab:memorycompare} and Table~\ref{tab:memoryusage}).

\textbf{Per-layer Weight Update.}  FO optimizers update all model parameters after BP by storing the entire gradients in memory. In contrast, ZO optimizers like SubZero calculate gradient estimates by first determining the loss value difference from two forward passes, then calculating the gradient estimate for each layer using this difference along with the layer's perturbation. To reduce memory usage during training, we can implement the parameter update with \texttt{optimizer.step()} after calculating the gradient estimate for each layer.

SubZero significantly reduces GPU memory consumption with the two implementation strategies. It should note that we use the per-layer weight update strategy for MeZO in all experiments.

To simplify hyperparameter tuning, we employ a norm alignment trick, allowing SubZero to directly utilize hyperparameter settings, such as the learning rate, from MeZO~\citep{malladi2023fine}. For a random perturbation matrix $\mZ \in \mathbb{R}^{m \times n}$, and its low-rank approximation is $\hat{\mZ} = \mU \mZ' \mV^\mathsf{T}$, where $\mU \in \mathbb{R}^{m \times r}$, $\mV \in \mathbb{R}^{n \times r}$, and $\mZ' \in \mathbb{R}^{r \times r}$. If $\mZ$ and $\mZ'$ are Gaussian random matrices, and $\mU$ and $\mV$ are column-orthogonal matrices, then we have:
\begin{align}
	\label{eqn:norm alignment}
	\mathbb{E}[\left \| \mZ \right \|_F] = \sqrt{\frac{m\times n}{r^2}} \mathbb{E}\left[ \| \hat{\mZ}\|_F\right].
\end{align}
Define $\mu = \sqrt{\frac{m \times n}{ r^2}}$. Let MeZO's learning rate be $\eta$ and perturbation scale be $\varepsilon$. There are two equivalent approaches to obtain the perturbation for SubZero. The first approach involves multiplying the random low-dimensional perturbation matrix by $\mu$, with SubZero adopting MeZO's hyperparameters directly: $\eta'=\eta$  and $\varepsilon'=\varepsilon$. The second approach keeps the random low-dimensional perturbation matrix fixed and sets SubZero's learning rate and perturbation scale as follows:
\begin{align*}
	\eta' = \eta \mu^2, \varepsilon' = \varepsilon \mu.
\end{align*}
We argue that norm alignment is crucial for SubZero, as changing the rank $r$ affects the norm of the gradient estimate, complicating the fine-tuning of the associated learning rate.

S-MeZO~\citep{liu2024sparse}, a new ZO method, aims to improve MeZO's performance and convergence speed. However, its source code and detailed layer-wise hyperparameter configurations have not been released.  \citet{yang2024adazeta} reproduce S-MeZO using a fixed sparsity ratio for each layer, selected based on the best overall result shown in Fig. 6 of their paper. So we perform S-MeZO with this non-official implementation code available at~\url{https://github.com/yifanycc/AdaZeta}.

\subsection{Datasets}

Following~\citep{malladi2023fine}, we use SuperGLUE~\citep{wang2019superglue} for OPT experiments, including BoolQ~\citep{clark2019boolq}, CB~\citep{demarneffe:cb}, COPA~\citep{roemmele2011choice}, MultiRC~\citep{khashabi2018looking}, ReCoRD~\citep{zhang2018record}, RTE~\citep{dagan2005pascal_rte1,bar2006second,giampiccolo2007third_rte3,bentivogli2009fifth_rte4}, WiC~\citep{pilehvar-camacho-collados-2019-wic}, and WSC~\citep{levesque2012winograd}.
We also utilize SST-2~\citep{socher2013recursive_sst-2} and two question answering (QA) datasets, SQuAD~\citep{rajpurkar-etal-2016-squad} and DROP~\citep{dua-etal-2019-drop}. For each task, we randomly sampled 1000 examples for training, 500 for validation, and 1000 for testing.

For LLama2-7B and Mistral-7B, we use CB~\citep{demarneffe:cb} in the full-parameter tuning and three PEFT schemes. For OPT-1.3B, we utilize SST-2~\citep{socher2013recursive_sst-2} in the full-parameter tuning and three PEFT schemes.

For RoBERTa-large, we consider classification datasets: SST-2~\citep{socher2013recursive_sst-2}, SST-5~\citep{socher2013recursive_sst-2}, MNLI~\citep{williams2018broad_mnli}, and SNLI~\citep{bowman2015large}.  Following~\cite{malladi2023fine}, the test set has 1000 examples for fast iteration, while we have 512 examples per class for both training and validation.

\begin{table}[t]	
	\caption{The hyperparameter search grids for OPT-13B. For each task, we run 20K steps for ZO methods (MeZO, S-MeZO, and SubZero) and SGD. We record the best model checkpoint based on the validation loss every 500 training steps.}
	\vspace{-1em}
	\begin{small}
	\label{tab:hyper_grid1}
	\begin{center}
		\renewcommand{\arraystretch}{0.9} 
		\begin{tabular}{lrc}
			\toprule
			Experiment & Hyperparameter & Value \\
			\midrule
			\multirow{3}{*}{MeZO(FT)} & batch size & 16 \\
			& learning rate & \{1e-7, 2e-7, 5e-7, 1e-6\} \\
			& $\varepsilon$ & 1e-3 \\
			\midrule
			\multirow{3}{*}{MeZO(LoRA)} & batch size & 16 \\
			& learning rate & \{1.5e-5, 3e-5, 5e-5\} \\
			& $\varepsilon$ & 1e-3 \\			
			\midrule
			\multirow{4}{*}{S-MeZO(FT)} & batch size & 16 \\
			& learning rate & \{1e-6, 5e-6\} \\
			& $\varepsilon$ & 1e-3 \\
			& sparse rate & 0.75 \\
			\midrule
			\multirow{4}{*}{S-MeZO(LoRA)} & batch size & 16 \\
			& learning rate & \{5e-5, 1e-4, 1e-3\} \\
			& $\varepsilon$ & 1e-3 \\
			& Sparse rate & 0.75 \\
			\midrule
			\multirow{5}{*}{SubZero(FT)} & batch size & 16 \\
			& learning rate & \{1e-7, 2e-7, 5e-7, 1e-6\} \\
			& $\varepsilon$ & 1e-3 \\
			& rank & \{32, 64, 128, 256 \} \\
			& subspace change frequency & \{500, 1000, 2000\} \\
			\midrule
			\multirow{5}{*}{SubZero(LoRA)} & batch size & 16 \\
			& learning rate & \{1.5e-5, 3e-5, 5e-5\} \\
			& $\varepsilon$ & 1e-3 \\
			& rank & \{4, 8, 16\} \\
			& subspace change frequency & \{500, 1000, 2000\} \\
			\midrule
			\multirow{2}{*}{SGD(FT)} & batch size & 16 \\
			& learning rate & \{1e-4, 1e-3, 5e-3\} \\
			
			\bottomrule
		\end{tabular}
	\end{center}
\end{small}
\end{table}

\subsection{Hyperparameters}
\label{sec: hyperparam}

Using a larger batch size can consistently reduce the variance in ZO optimization, thus enhancing fine-tuning performance~\citep{malladi2023fine, gautamvariance, yang2024adazeta}. However, this increase in batch size also raises the time for forward passes and significantly elevates memory usage. We focus on developing ZO methods that minimize variance and improve performance with small batch sizes, with a default setting of 16. In some SGD experiments, like on MultiRC and SQuAD, the batch size is reduced to 8 due to limited GPU resources.

Consistent with previous studies~\citep{malladi2023fine, zhang2024revisiting, liu2024sparse, yang2024adazeta}, we employ SGD without momentum by default to maintain memory efficiency. SGD utilizes the linear learning rate schedule, while all ZO methods with SGD apply a constant learning rate schedule, with weight decay set to 0.

For the projection matrix generation experiments summarized in Table~\ref{tab:proj_choice}, we perform full-parameter fine-tuning on the SST-2 dataset using three models: the OPT-1.3B model with a rank of 24 and a subspace update frequency of 1000, the LLaMA2-7B model with a rank of 24 and a subspace update frequency of 1000 and the OPT-13B model with a rank of 128 and a subspace update frequency of 1000. All other hyperparameters, including learning rate and perturbation scale, remain consistent across experiments. While the decomposition of the weight and Gaussian random matrices is relatively straightforward, the activation matrix relies on the AdaBK algorithm~\citep{yong2023general}. Specifically, we compute the second-order statistics of activations (see line 7 of Algorithm 1 in the paper of AdaBK) and employ an adaptive damping strategy based on the maximum singular value to improve the condition number, followed by QR decomposition. For dimensionality alignment, the second-order statistics of output and input features are decomposed to derive $\mathbf{U}$ and $\mathbf{V}$, respectively. Finally, for the decomposition of the historical zeroth-order gradient, its matrix from the previous iteration is readily obtained and subsequently QR-decomposed, facilitated by our random seed technique and the saved historical projection matrix.

For RoBERTa, we run Adam for 1K steps and ZO methods for 100K steps. In the rest experiments, we run Adam for 5 epochs and SGD and ZO methods for 20K steps.
 
We follow previous work to set the hyperparameters in PEFT schemes~\citep{malladi2023fine, zhang2024revisiting}. For LoRA, the rank is set to 8 and $\alpha$ is set to 16. For prefix tuning, the length of prefix tokens is set to 5, and we initialize these tunable representations by randomly sampling tokens from the vocabulary and then passing them through the LLM to get their keys and values at different attention
layers. For prompt tuning, the length of prompt virtual tokens is set to 10, and the prompt tokens are initialized with actual token values from the model's embedding.

We present the hyperparameter search grids in Tables~\ref{tab:hyper_grid1} and~\ref{tab:hyper_grid2} to assist with result reproduction. For OPT-1.3B, we utilize the same hyperparameter settings as in Table~\ref{tab:hyper_grid2}. For RoBERTa-large, we use a learning rate of \{1e-6, 5e-6\} and $\varepsilon$=1e-3 for MeZO and SubZero,  with a  batch size of 64. The rank for SubZero is set to \{8, 16, 24\}, and subspace change frequency is adjusted to \{1000, 2000\}. 

For the ablation study, we evaluate the effectiveness of the orthogonal projection matrix using the OPT-13B model in full-parameter tuning scheme on the RTE and WSC datasets, and the results are presented in Table~\ref{tab:abla_projtype}. The hyperparameter settings are consistent with those in Table~\ref{tab:opt13b_performance}, and further details are available in Table~\ref{tab:hyper_grid1}. The subspace dimensionality remains fixed across all experiments. It is noteworthy that both orthogonal and non-orthogonal projection matrices can utilize the same learning rate and perturbation scale. This is because the overall perturbation matrix is scaled by a factor of $\frac{1}{r}$, following a similar norm alignment strategy as detailed in Eqn.~\eqref{eqn:norm alignment}. We also perform ablation studies on the rank and subspace update frequency for  SubZero, with results shown in Table~\ref{tab:abla-rank-T}. Full-parameter tuning scheme is conducted on the RTE dataset using the OPT-13B model, with specific experimental settings outlined in Table~\ref{tab:hyper_grid1}. All experiments employ the same learning rate and perturbation scale, enabled by the norm alignment technique described in Eqn.~\eqref{eqn:norm alignment}.

\begin{table}[t]	
	\caption{The hyperparameter search grids for LLama2-7B and Mistral-7B. For each task, we run 20K steps for ZO methods (MeZO and SubZero) and SGD. We record the best model checkpoint based on the validation loss every 500 training steps.}
	\vspace{-1em}
	\label{tab:hyper_grid2}
	\begin{small}
	\begin{center}		
		\renewcommand{\arraystretch}{0.9} 
		\begin{tabular}{lrc}
			\toprule
			Experiment & Hyperparameter & Value \\
			\midrule
			\multirow{3}{*}{MeZO(FT)} & batch size & 16 \\
			& learning rate & \{1e-7, 5e-7, 1e-6\} \\
			& $\varepsilon$ & 1e-3 \\
			\midrule
			\multirow{3}{*}{MeZO(LoRA)} & batch size & 16 \\
			& learning rate & \{1e-6, 5e-6, 1e-5, 3e-5\} \\
			& $\varepsilon$ & 1e-3 \\
			\midrule
			\multirow{3}{*}{MeZO(Prefix)} & batch size & 16 \\
			& learning rate & \{1e-3, 5e-3, 1e-2\} \\
			& $\varepsilon$ & 1e-1 \\
			\midrule
			\multirow{3}{*}{MeZO(Prompt)} & batch size & 16 \\
			& learning rate & \{1e-3, 5e-3, 1e-2\} \\
			& $\varepsilon$ & 1e-2 \\
			\midrule
			\multirow{5}{*}{SubZero(FT)} & batch size & 16 \\
			& learning rate & \{1e-7, 5e-7, 1e-6\} \\
			& $\varepsilon$ & 1e-3 \\
			& rank & \{24, 48\} \\
			& subspace change frequency & 1000 \\
			\midrule
			\multirow{5}{*}{SubZero(LoRA)} & batch size & 16 \\
			& learning rate & \{1e-6, 5e-6, 1e-5, 3e-5\} \\
			& $\varepsilon$ & 1e-3 \\
			& rank & \{4, 8\} \\
			& subspace change frequency & 1000 \\
			\midrule
			\multirow{5}{*}{SubZero(Prefix)} & batch size & 16 \\
			& learning rate & \{1e-3, 5e-3, 1e-2\} \\
			& $\varepsilon$ & 1e-1 \\
			& rank & \{4, 8\} \\
			& subspace change frequency & 1000 \\
			\midrule
			\multirow{5}{*}{SubZero(Prompt)} & batch size & 16 \\
			& learning rate & \{1e-3, 5e-3, 1e-2\} \\
			& $\varepsilon$ & 1e-2 \\
			& rank & \{16, 24\} \\
			& subspace change frequency & 1000 \\
			\midrule
			\multirow{2}{*}{SGD(FT)} & batch size & 16 \\
			& learning rate & \{1e-5, 1e-4, 1e-3, 5e-3\} \\
			\bottomrule
		\end{tabular}
	\end{center}
\end{small}
\end{table}

\subsection{Prompt Templates}
\label{sec:prompt}

For autoregressive LLMs, we have three task types: classification, multiple-choice, and question answering. We adopt the prompt templates for various tasks in~\citep{malladi2023fine}, which are summarized in Table~\ref{tab:task_prompt}. For masked LLMs, we also adopt the prompt templates in~\citep{malladi2023fine} and present them in Table~\ref{tab:dataset_statistics}.

\begin{table}[htbp]	
	\caption{The prompt templates used in the OPT-1.3B, OPT-13B, LLama2-7B, and Mistral-7B experiments.}
	\vspace{-1em}
	\begin{small}
	\label{tab:task_prompt}	
	\begin{center}
			\renewcommand{\arraystretch}{0.9} 
			\begin{tabular}{lll}
				\toprule
				{Task} & {Type} & {Prompt} \\
				\midrule
				SST-2 & cls. & <text> It was \textcolor{blue}{terrible/great} \\
				RTE & cls. & <premise> \\
				& & Does this mean that "<hypothesis>" is true? Yes or No? \\
				&& \textcolor{blue}{Yes or No} \\
				CB & cls. & Does this mean that "<hypothesis>" is true? Yes or No? \\
				&& \textcolor{blue}{Yes/No/Maybe} \\
				BoolQ & cls. & <passage> <question>? \\
				&& \textcolor{blue}{Yes/No} \\
				WSC & cls. & <text> \\
				&& In the previous sentence, does the pronoun "<span2>" refer to <span1>? Yes or No? \\
				&& \textcolor{blue}{Yes/No} \\
				WIC & cls. & Does the word "<word>" have the same meaning in these two sentences? Yes, No? \\
				&&<sentence1> \\
				&& <sentence2> \\
				&& \textcolor{blue}{Yes/No} \\
				MultiRC & cls. & <paragraph> \\
				&& Question: <question> \\
				&& I found this answer "<answer". Is that correct? Yes or No? \\
				&& \textcolor{blue}{Yes/No} \\
				COPA & mch. & <premise> so/because <candidate> \\
				ReCoRD & mch. & <passage> \\
				&& <query>.replace("@placeholder", <candidate>) \\
				SQuAD & QA & Title: <title> \\
				&& Context: <context> \\
				&& Question: <question> \\
				&& Answer: \\
				DROP & QA & Passage: <context> \\
				&& Question: <question> \\
				&& Answer: \\				
				\bottomrule
		\end{tabular}
	\end{center}
\end{small}
\end{table}

\begin{table}[ht]
	\caption{The prompt templates used in RoBERTa-large experiments. $C$ is the number of classification categories.}
	\begin{small}
	\centering
	\renewcommand{\arraystretch}{0.9} 
		\begin{tabular}{llll}
			\toprule
			{Task} & $C$  & {Type} & {Prompt} \\
			\midrule
			SST-2 & 2 & sentiment  cls.& <sentence1> It was \textcolor{blue}{great/terrible} \\
			SST-5 & 5  & sentiment cls.&  <sentence1> It was \textcolor{blue}{great/good/okay/bad/terrible}  \\
			MNLI & 3  &  NLI & <sentence1> ? \textcolor{blue}{Yes/Maybe/No} , <sentence2>  \\
			SNLI & 3  &  NLI & <sentence1> ? \textcolor{blue}{Yes/Maybe/No} , <sentence2>  \\
			\bottomrule
		\end{tabular}
	
	\label{tab:dataset_statistics}
	\end{small}
\end{table}

\subsection{Proofs}
\label{appendix: proofs}


In practice, SubZero employs smaller and layer-specific low-rank perturbation matrices instead of a large model-scale projection matrix. However, it is more convenient to prove SubZero's properties using a model-scale projection. Fortunately, the following lemma shows that the low-rank perturbation matrix for each layer can be represented as a layer-scale projection matrix, which is column orthogonal.
\begin{restatable}{lemma}{lemmauvproj}
	\label{lem:UV-proj}
	Let $\tilde{\mZ}=\mU\mZ\mV^{\mathsf{T}}$, where $\mU\in\R^{m\times r}, \mZ\in\R^{r\times r}, \mV\in\R^{n\times r}$, and $\mU^{\mathsf{T}}\mU=\mV^{\mathsf{T}}\mV=\mI_r$. Then we have ${\rm vec}(\tilde{\mZ}) = \mP{{\rm vec}}(\mZ)$ and $\mP^{\mathsf{T}}\mP=\mI_{r^2}$, where $\mP=\mV \otimes \mU$. 
\end{restatable}

\begin{proof}
	Since ${\rm vec}(\mU\mZ\mV^{\mathsf{T}}) = (\mV \otimes \mU){\rm vec}(\mZ)$, we only need to show $(\mV \otimes \mU)^{\mathsf{T}}(\mV \otimes \mU)=\mI_{r^2}$. In fact
	\begin{align*}
		(\mV \otimes \mU)^{\mathsf{T}}(\mV \otimes \mU) = (\mV^{\mathsf{T}} \otimes \mU^{\mathsf{T}})(\mV \otimes \mU) = (\mV^{\mathsf{T}}\mV) \otimes (\mU^{\mathsf{T}}\mU) = \mI_r \otimes \mI_r = \mI_{r^2}.
	\end{align*}
	The proof is completed.
\end{proof}


We can also demonstrate that the low-rank perturbation matrices across all layers can be represented as a model-scale projection matrix. We first give the following lemma.
\begin{restatable}{lemma}{lemmablockdiag}
	\label{lem:block-diag}
	Let a block diagonal matrix $\mP = {\rm bdiag}(\mP_1, \mP_2, \cdots, \mP_l)$ and $\tilde{\vz}_i=\mP_i\vz_i$, where $\mP_i^{\mathsf{T}}\mP_i=\mI_{r^2}$ and $i=1,2,\dots,l$. Then we have $\tilde{\vz}=\mP\vz$, where $\tilde{\vz} = [\tilde{\vz}_1^{\mathsf{T}}, \dots, \tilde{\vz}_l^{\mathsf{T}}]^{\mathsf{T}}$, $\vz = [\vz_1^{\mathsf{T}}$, $\dots, \vz_l^{\mathsf{T}}]^{\mathsf{T}}$ and $\mP^{\mathsf{T}}\mP=\mI_{lr^2}$.
\end{restatable}

\begin{proof}
	It is easy to check that $\tilde{\vz}=\mP\vz$. Besides, we have
	\begin{align*}
		\mP^{\mathsf{T}}\mP = {\rm bdiag}(\mP_1^{\mathsf{T}}, \dots, \mP_l^{\mathsf{T}}){\rm bdiag}(\mP_1, \dots, \mP_l) = {\rm bdiag}(\mP_1^{\mathsf{T}}\mP_1, \dots, \mP_l^{\mathsf{T}}\mP_l) = \mI_{lr^2}.
	\end{align*}
	The proof is completed.
\end{proof}

We may define $\mP = {\rm bdiag} (\boldsymbol{V}_1 \otimes \boldsymbol{U}_1, \boldsymbol{V}_2 \otimes \boldsymbol{U}_2, \cdots, \boldsymbol{V}_l \otimes \boldsymbol{U}_l )$ that satisfies $\mP^\mathsf{T}\mP = \mI$, $\vz = [{\rm vec}(\mZ_1)^{\mathsf{T}}, {\rm vec}(\mZ_2)^{\mathsf{T}}, \dots, {\rm vec}(\mZ_l)^{\mathsf{T}}]^{\mathsf{T}}$, and $\tilde{\vz} = [{\rm vec}(\tilde{\mZ}_1)^{\mathsf{T}}, {\rm vec}(\tilde{\mZ}_2)^{\mathsf{T}}, \dots, {\rm vec}(\tilde{\mZ}_l)^{\mathsf{T}}]^{\mathsf{T}}$. Then according to Lemma~\ref{lem:block-diag}, the perturbation vector of SubZero is $\tilde{\boldsymbol{z}} = \mP \boldsymbol{z}$, which is similar as existing random subspace methods in Eqn.~\eqref{eqn:subspace_spsa}, but with SubZero's projection matrix being block diagonal and column orthogonal. 

To prove Theorem~\ref{thm:oracles-mean} and Theorem~\ref{thm:oracles-quadratic}, we first introduce some definitions and lemmas about Gaussian distribution.

\begin{defination}
	We say $\boldsymbol{z}$ is a standard $n$-dimensional Gaussian vector (denote by $\boldsymbol{z} \sim \mathcal{N} (\boldsymbol{0} , \mI_n)$), if its probability density function $p(\boldsymbol{z}) = \frac{1}{\kappa}e^{-\frac{1}{2}\| \boldsymbol{z} \|^2}$, where $\kappa > 0$ satisfies $\int_{\R^n} \frac{1}{\kappa}e^{-\frac{1}{2}\| \boldsymbol{z} \|^2} d\boldsymbol{z}=1$.
\end{defination}

\begin{defination}
	Let $\boldsymbol{z} \sim \mathcal{N} (\boldsymbol{0} , \mI_n)$. We say $x$ is a chi-square random variable with degrees of freedom $n$ (denote by $x \sim \chi^2(n)$), if $x = \| \boldsymbol{z} \|^2$.
\end{defination}

\begin{restatable}{lemma}{lemmagaussian}
	\label{lem:gaussian}
	Let $\boldsymbol{z} \sim \mathcal{N} (\boldsymbol{0} , \mI_n)$. For any orthogonal $(n \times n)$-matrix $\mQ$ and continuous function $f$, we have $\E_{\boldsymbol{z}} [ f(\boldsymbol{z}) ] = \E_{\boldsymbol{z}} [ f(\mQ\boldsymbol{z}) ]$.
\end{restatable}

\begin{restatable}{lemma}{lemmachiF}
	\label{lem:chi-F}
	If $x \sim \chi^2(n)$, then we have
	\begin{align*}
		\E_{x} [ x ] = n, \quad \Var_{x}[ x ] = 2n.
	\end{align*}
\end{restatable}

\begin{restatable}{lemma}{lemmaC22func} \citep{nesterov2017random}
	\label{lem:C22-func}
	Let $f\in C^{2,2}_{L_2}(\R^n)$. Then for all $\vx, \vy \in \R^n$, we have
	\begin{align*}
		| f(\vy)-f(\vx)-\langle \nabla f(\vx), \vy-\vx \rangle - \frac{1}{2}\langle \nabla^2f(\vx)(\vy-\vx), \vy-\vx\rangle | \le \frac{L_2}{6} \| \vy-\vx \|^3.
	\end{align*}
\end{restatable}

\begin{restatable}{lemma}{lemmagaussnorm1} \citep{nesterov2017random}
	\label{lem:gaussnorm1}
	Let $\boldsymbol{z} \sim \mathcal{N} (\boldsymbol{0} , \mI_n)$. For $0\le t \le 2$, we have
	\begin{align*}
		\E_{\boldsymbol{z}} [ \| \boldsymbol{z} \|^t ] \le n^{t/2}.
	\end{align*}
	For $t \ge 2$, we have
	\begin{align*}
		n^{t/2} \le \E_{\boldsymbol{z}} [ \| \boldsymbol{z} \|^t ] \le (n+t)^{t/2}.
	\end{align*}
\end{restatable}

\begin{restatable}{lemma}{lemmagaussnorm2}
	\label{lem:gaussnorm2}
	Let $\boldsymbol{z} \sim \mathcal{N} (\boldsymbol{0} , \mI_n)$. For all $\vy \in \R^n$, we have
	\begin{align*}
		\E_{\boldsymbol{z}} [ \| \langle \vy, \boldsymbol{z} \rangle \boldsymbol{z} \|^2 ] = (n+2) \| \vy \|^2.
	\end{align*}
\end{restatable}

\begin{proof}
	Note that for any orthogonal $(n \times n)$-matrix $\mQ$, we have
	\begin{align*}
		\| \langle \vy,\mQ\boldsymbol{z} \rangle \mQ\boldsymbol{z} \|^2 = \| \langle \mQ^{\mathsf{T}}\vy, \boldsymbol{z} \rangle\boldsymbol{z} \|^2, \quad \| \mQ^{\mathsf{T}}\vy \| = \| \vy \|.
	\end{align*}
	In accordance with Lemma~\ref{lem:gaussian}, we can set $\vy=[1, 0, \dots, 0]^{\mathsf{T}}$, and only need to prove $\E_{\boldsymbol{z}} [ \| \langle \vy, \boldsymbol{z} \rangle \boldsymbol{z} \|^2 ] = n+2$.
	Equipped with Lemma~\ref{lem:chi-F}, we get
	\begin{align*}
		\E_{\boldsymbol{z}} [ \| \langle \vy, \boldsymbol{z} \rangle \boldsymbol{z} \|^2 ] = \E_{\boldsymbol{z}} \left[ \sum_{i=1}^n\boldsymbol{z}_1^2\boldsymbol{z}_i^2 \right] = \sum_{i=1}^n  \E_{\boldsymbol{z}}[\boldsymbol{z}_1^2\boldsymbol{z}_i^2] = \E_{\boldsymbol{z}_1}[\boldsymbol{z}_1^4] + \E_{\boldsymbol{z}_1}[\boldsymbol{z}_1^2] \sum_{i=2}^n  \E_{\boldsymbol{z}}[\boldsymbol{z}_i^2] = n+2.
	\end{align*}
	The proof is completed.
\end{proof}

\theoremmean*

\begin{proof}
	 \textbf{a)} Evidently, the conclusion is established based on Lemma~\ref{lem:UV-proj} and Lemma~\ref{lem:block-diag}.
	
	\textbf{b)} Let $a_{\vz}(\tau)=f(\vx+\tau\vz)-f(\vx)-\tau\langle \nabla f(\vx),\vz \rangle - \frac{\tau^2}{2} \langle \nabla^2f(\vx)\vz,\vz \rangle$. Lemma~\ref{lem:C22-func} implies that
	\begin{align*}
		| a_{\vz}(\pm\varepsilon) | \le \frac{\varepsilon^3}{6}L_2\|\vz\|^3.
	\end{align*}
	Note that
	\begin{align*}
		& \mathbb{E}_{\boldsymbol{z}}[ \hat{g}_{\varepsilon}(\vx, \mP, \boldsymbol{z}) ] - \mP\mP^{\mathsf{T}}\nabla f(\vx) \\ & = \frac{\mP}{2\kappa\varepsilon} \int_{\R^q} [f(\vx+\varepsilon\mP\vz)-f(\vx-\varepsilon\mP\vz)-2\varepsilon\langle \nabla f(\vz),\mP\vz \rangle] \vz e^{-\frac{1}{2}\| \vz \|^2} d\vz.
	\end{align*}
	Therefore, in accordance with Lemma~\ref{lem:gaussnorm1}, we have
	\begin{align*}
		& \| \mathbb{E}_{\boldsymbol{z}}[ \hat{g}_{\varepsilon}(\vx, \mP, \boldsymbol{z}) ] - \mP\mP^{\mathsf{T}}\nabla f(\vx) \| \\ 
		& \le \frac{1}{2\kappa\varepsilon} \int_{\R^q} |f(\vx+\varepsilon\mP\vz)-f(\vx-\varepsilon\mP\vz)-2\varepsilon\langle \nabla f(\vz),\mP\vz \rangle| \|\vz\| e^{-\frac{1}{2}\| \vz \|^2} d\vz \\
		& = \frac{1}{2\kappa\varepsilon} \int_{\R^q} | a_{\mP\vz}(\varepsilon)-a_{\mP\vz}(-\varepsilon) | \|\vz\| e^{-\frac{1}{2}\| \vz \|^2} d\vz \\
		& \le \frac{\varepsilon^2 L_2}{6\kappa} \int_{\R^q} \|\vz\|^4 e^{-\frac{1}{2}\| \vz \|^2}  d\vz \le \frac{\varepsilon^2}{6}L_2(q+4)^2.
	\end{align*}
	The proof is completed.
\end{proof}

\theoremquadratic*

\begin{proof}
	It is easy to check that $\hat{g}_{\varepsilon}(\vx, \mP, \boldsymbol{z}) = \mP\langle \mP^{\mathsf{T}}\nabla f(\vx),\boldsymbol{z} \rangle \boldsymbol{z}$. Thus we have $\mathbb{E}_{\boldsymbol{z}}[ \hat{g}_{\varepsilon}(\vx, \mP, \boldsymbol{z}) ] = \mP\mP^{\mathsf{T}}\nabla f(\vx)$. Combined with Lemma~\ref{lem:gaussnorm2}, we get $	\mathbb{E}_{\boldsymbol{z}}[ \| \hat{g}_{\varepsilon}(\vx, \mP, \boldsymbol{z}) \|^2 ] = (q+2) \| \mP^{\mathsf{T}}\nabla 	f(\vx) \|^2$. Note that for any orthogonal $(q \times q)$-matrix $\mQ$, we have
	\begin{align*}
		\mathbb{E}_{\boldsymbol{z}} \left[ \frac{\langle \nabla f(\vx), \hat{g}_{\varepsilon}(\vx, \mP, \boldsymbol{z}) 	\rangle^2}{\| \mP^{\mathsf{T}}\nabla f(\vx) \|^2 \| \hat{g}_{\varepsilon}(\vx, \mP, \boldsymbol{z}) \|^2} \right] & = \mathbb{E}_{\boldsymbol{z}} \left[ \frac{\langle \mP^{\mathsf{T}}\nabla f(\vx),\boldsymbol{z} \rangle^2}{\| \mP^{\mathsf{T}}\nabla f(\vx) \|^2\| \boldsymbol{z} \|^2} \right] \\
		& = \mathbb{E}_{\boldsymbol{z}} \left[ \frac{\langle \mP^{\mathsf{T}}\nabla f(\vx),\mQ\boldsymbol{z} \rangle^2}{\| \mP^{\mathsf{T}}\nabla f(\vx) \|^2\| \mQ\boldsymbol{z} \|^2} \right] \\
		& = \mathbb{E}_{\boldsymbol{z}} \left[ \frac{\langle \mQ^{\mathsf{T}}\mP^{\mathsf{T}}\nabla f(\vx),\boldsymbol{z} \rangle^2}{\| \mQ^{\mathsf{T}}\mP^{\mathsf{T}}\nabla f(\vx) \|^2\| \boldsymbol{z} \|^2} \right] .
	\end{align*}
	In accordance with Lemma~\ref{lem:gaussian}, we can set $\mP^{\mathsf{T}}\nabla f(\vx)=[1, 0, \dots, 0]^{\mathsf{T}}$. Thus we have
	\begin{align*}
		\mathbb{E}_{\boldsymbol{z}} \left[ \frac{\langle \nabla f(\vx), \hat{g}_{\varepsilon}(\vx, \mP, \boldsymbol{z}) 	\rangle^2}{\| \mP^{\mathsf{T}}\nabla f(\vx) \|^2 \| \hat{g}_{\varepsilon}(\vx, \mP, \boldsymbol{z}) \|^2} \right] = \mathbb{E}_{\boldsymbol{z}} \left[ \frac{\boldsymbol{z}_1^2}{\|\boldsymbol{z}\|^2} \right] = \frac{1}{q}.
	\end{align*}
	The proof is completed.
\end{proof}


To illustrate the convergence of Subzero with SGD, our analysis is divided into two main segments. We first investigate the convergence behavior of SubZero solution process while keeping the projection matrix \(\mP\) constant. Next, we evaluate the effects of the lazy updates to  \(\mP\). Based on these evaluations, we establish the global convergence of Subzero. Without loss of generality, we concentrate on the scenario where the number of layers is 1.
 
First, when the subspace $\mP$ is fixed, the original problem of SubZero can be reformulated as an optimization problem within the subspace. Define $h(\vy) = f(\vx + \mP\vy)$, $h_\varepsilon(\vy) = \mathbb{E}_{\vz}[h(\vy+\varepsilon \vz)]$, and  $g_\varepsilon (\vy)= \frac{h(\vy + \varepsilon \boldsymbol{z})-f(\vy)}{\varepsilon}\boldsymbol{z} $. According to Lemma~\ref{lem:preserveconvexandsmooth}, if $f$ is first $L_1$-smooth, then $h$ is also first $L_1$-smooth.

\begin{restatable}{lemma}{lemmapreserveconvexandsmooth}
	\label{lem:preserveconvexandsmooth}
	Let $h(\vy) = f(\vx + \mP\vy)$, where $f\in C^{1,1}_{L_1}(\mathbb{R}^d)$, and $\mP^\mathsf{T}\mP = \mI$, then we have $h \in C^{1,1}_{L_1}(\mathbb{R}^q)$.
\end{restatable}

\begin{proof}
	
	The following proves that if $f$ is first $L_1$-smooth, then $h$ is also first $L_1$-smooth. For any $\vy_1 \in \mathbb{R}^q$ and $\vy_2 \in \mathbb{R}^q$, we have
	\begin{align*}
		\left \| \nabla h(\vy_1) - \nabla h(\vy_2) \right \| &= \left \| \mP^\mathsf{T} \nabla(f(\vx + \mP \vy_1) - \boldsymbol{P}^\mathsf{T} \nabla(f(\vx + \mP \vy_2)\right \|  \\
		& \le \left \| \mP^\mathsf{T}\right \| \left \|\nabla(f(\vx + \mP \vy_1) -  \nabla(f(\vx + \mP \vy_2) \right \| \\
		& \le  L_1 \left \| \mP(\vy_1 - \vy_2) \right \| \\
		&= L_1 \left \|\vy_1 -\vy_2 \right \|.
	\end{align*}
	
	The proof is completed. 	
\end{proof}

Now, we can analyze the convergence of SubZero when fixing the subspace.

\begin{restatable}{lemma}{lemmanestlem5} \cite{nesterov2017random}
	\label{len:nestlem5}
	Let $f\in C_{L_1}^{1,1}(\mathbb{R})$. Then, for any $\vx \in \mathbb{R}$, we have 
	\begin{align}
		E_\vz[\|g_\varepsilon(\vx)\|^2] &= E_\vz\left[\|\frac{f(\vx + \varepsilon \vz) - f(\vx)}{\varepsilon}\|^2\right] \leq 4(n+4)\|\nabla f_\varepsilon(\vx)\|^2+3\varepsilon^2L_1^2(f)(n+4)^3,
	\end{align}
	and
	\begin{align}
		\|\nabla f(\vx)\|^2\leq2\|\nabla f_\varepsilon(\vx)\|^2+\frac{\varepsilon^2}{2}L_1^2(f)(n+6)^3,
	\end{align}
    where $f_\varepsilon(\vx) = \mathbb{E}_\vz[f(\vx + \varepsilon\vz)]$.
\end{restatable}

\begin{restatable}{lemma}{lemmasubspaceconvergence}
	\label{lem:sub_convergence}
	Let $\vy^*$ = $\argmin_{\vx \in \mathbb{R}^q}h(\vy)$, where $h \in C^{1,1}_{L_1}(\mathbb{R}^q)$ and $h$ is non-convex. Suppose $\mathcal{E}_{k} = (\vz_0, \vz_1, \cdots, \vz_{k-1}, \vz_{k})$, where $\vz_{k} \sim \mathcal{N}(0, \mI_q)$ and $\eta = \frac{1}{4(q+4)L_1}$. $\{\vy_{k}\}_{k>0}$ is the sequence generated by Algorithm~\ref{alg:ss_mezo}. Let \( \phi_{0} = h(\vy_{0}) \), and for \( k \geq 1 \), \( \phi_{k} = \mathbb{E}_{\mathcal{E}_{k-1}}[h(\vy_{k})] \). For the $\boldsymbol{P}$ defined in~\eqref{eq:def_P}, which is fixed,  we have 
	\begin{align}
		\phi_{k+1} - \phi_{k} \leq  - \frac{1}{4} \eta \mathbb{E}_{\mathcal{E}_{k}}\left[\|\nabla h(\vy_{k})\|^{2}\right]  + \frac{\varepsilon^{2}(q+6)^{3}}{8}L_{1}^{2}+ \frac{3\varepsilon^{2}(q+4)}{32}L_{1}
	\end{align}
\end{restatable}

\begin{proof}
	
	If a subspace \( \mP \in \mathbb{R}^{d \times q} \) is fixed, the optimization objective can be reformulated as
	\[
	\min_{\vy \in \mathbb{R}^{q}} h(\vy) := f(\vx + \mP \vy),
	\]
	
	 Let \( \vy_{0} \) be an initial point and \( \{\eta_{k}\}_{k \geq 0} \) a sequence of positive real numbers. Consider the randomized gradient search algorithm \( \mathcal{R} \mathcal{G}_{\varepsilon} (\varepsilon > 0) \):
	
	1) Generate \( \vz_{k} \) and the corresponding \( g_{\varepsilon}(\vy_{k}) \), where \( \vz_{k} \sim \mathcal{N}(\vzero, \mI_{q}) \).
	
	2) Update \( \vy_{k+1} = \vy_{k} - \eta_k g_{\varepsilon}(\vy_{k}) \).
	
	We aim to estimate the evolution of the function \( h_{\varepsilon} \) after one iteration of this algorithm.
	
	Given that \( h \) is \( L_1 \)-Lipschitz continuous for the first derivative, and \( h_{\varepsilon} \) is \(L_\varepsilon\)-Lipschitz continuous for the first derivative (where \(L_\varepsilon \leq L_1\))\citep{nesterov2017random}. Thus, we have
	\[
	h_{\varepsilon}(\vy_{k+1}) \le h_{\varepsilon}(\vy_{k}) - \eta_{k} \langle \nabla h_{\varepsilon}(\vy_{k}), g_{\varepsilon}(\vy_{k}) \rangle + \frac{1}{2} \eta_{k}^{2} L_\varepsilon\|g_{\varepsilon}(\vy_{k})\|^{2}.
	\]
	Taking expectation with respect to \( \vz_{k} \), we obtain
	\[
	\mathbb{E}_{\vz_{k}}[h_{\varepsilon}(\vy_{k+1})] \le h_{\varepsilon}(\vy_{k}) - \eta_{k} \|\nabla h_{\varepsilon}(\vy_{k})\|^{2} + \frac{1}{2} \eta_{k}^{2} L_{\varepsilon} \, \mathbb{E}_{\vz_{k}}[\|g_{\varepsilon}(\vy_{k})\|^{2}].
	\]
	Since \( h \in C^{1,1}(\mathbb{R}^{q}) \), from Lemma~\ref{len:nestlem5}, we have 
	\[
	\begin{aligned}
		\mathbb{E}_{\vz_{k}}[h_{\varepsilon}(\vy_{k+1})] &\le h_{\varepsilon}(\vy_{k}) - \eta_{k} \|\nabla h_{\varepsilon}(\vy_{k})\|^{2} \\
		&\quad + \frac{1}{2} \eta_{k}^{2} L_{1} \left(4(q + 4) \|\nabla h_{\varepsilon}(\vy_{k})\|^{2} + 3 \varepsilon^{2} L_{1}^{2}(q + 4)^{3}\right).
	\end{aligned}
	\]
	Setting \( \eta_{k} = \hat{\eta} = \frac{1}{4(q + 4)L_{1}} \), we get
	\[
	\begin{aligned}
		\mathbb{E}_{\vz_{k}}[h_{\varepsilon}(\vy_{k+1})] &\le h_{\varepsilon}(\vy_{k}) - \frac{1}{2} \hat{\eta} \|\nabla h_{\varepsilon}(\vy_{k})\|^{2} + \frac{3 \varepsilon^{2}}{32} L_{1}(q + 4).
	\end{aligned}
	\]
	Taking the expectation with respect to \( \mathcal{E}_{k} \), we get
	\[
	\phi_{k+1} \le \phi_{k} - \frac{1}{2} \hat{\eta} \mathbb{E}_{\mathcal{E}_{k}}[\|\nabla h_{\varepsilon}(\vy_{k})\|^{2}] + \frac{3 \varepsilon^{2}(q + 4)}{32} L_{1},
	\]
	From Lemma~\ref{len:nestlem5}, we have \(  \mathbb{E}_{\mathcal{E}_{k}}[\|\nabla h(\vy_{k})\|^{2}] \le 2 \mathbb{E}_{\mathcal{E}_{k}}[\|\nabla h_{\varepsilon}(\vy_{k})\|^{2}] + \frac{\varepsilon^{2} (q + 6)^{3}}{2} L_{1}^{2} \).
Therefore,
	\begin{align}
		\phi_{k+1} - \phi_{k} \leq -\frac{1}{4} \hat{\eta} \mathbb{E}_{\mathcal{E}_{k}}\left[\|\nabla h(\vy_{k})\|^{2}\right]  + \frac{\varepsilon^{2}(q+6)^{3}}{8}L_{1}^{2}+ \frac{3\varepsilon^{2}(q+4)}{32}L_{1}.
	\end{align}
	The proof is completed. 	
\end{proof}

Next, we need to measure the randomness of our random subspace. From Lemma~\ref{lem:randomproj3}, if the projection matrix is obtained by Algorithm~\ref{alg:getortho}, we have $\mathbb{E}[\mP \mP^T] = \frac{q}{d} \mI$, where $q$ represents the dimension of the subspace, $d$ represents the dimension of the origin space, and $\mP = \mV \otimes \mU$ (see Lemma~\ref{lem:UV-proj}).

\begin{restatable}{lemma}{lemmarandomproj0}
	\label{lem:randomproj0}
	Let matrix $\mA = (\va_1, \va_2, \cdots, \va_r) \in \mathbb{R}^{n \times r}$ be composed of column vectors $\va_k$ which are mutually independent and $\va_k \in \mathcal{N}(0, \mI_n)$. Suppose Gram-Schmidt process $\vu_k = \va_k - \sum_{s=1}^{k-1}\left \langle  \va_k, \ve_s\right \rangle \ve_s $ and $\ve_k = \frac{\vu_k}{\left \| \vu_k \right \| }$. $[\va_k]_i\leftrightarrow [\va_k]_j$ represents the exchange of the $i$-th element and the $j$-th element of $\va_k$, while all other elements remain unchanged. $[\va_k]_i = -1 \times [\va_k]_i$ signifies that only the $i$-th element of $\va_k$ is multiplied by $-1$, while all other elements remain unchanged. Suppose $f(\mA, \mU, \mE)$ be a function of the matrix $\mA$, $\mU = (\vu_1, \vu_2, \cdots, \vu_r)$ and $\mE = (\ve_1, \ve_2, \cdots, \ve_r)$, then 
    
    (1) if $[\va_k]_i\leftrightarrow [\va_k]_j$ or $[\va_k]_i = -1 \times [\va_k]_i$, $\mathbb{E}[f]$ remain unchanged.
    
    (2) if $[\va_k]_i\leftrightarrow [\va_k]_j \Rightarrow [\vu_k]_i\leftrightarrow [\vu_k]_j$ and $[\ve_k]_i\leftrightarrow [\ve_k]_j$.

    (3) if  $[\va_k]_i = -1 \times [\va_k]_i\Rightarrow [\vu_k]_i = -1 \times [\vu_k]_i$ , $[\ve_k]_i = -1 \times [\ve_k]_i$, $[\vu_k]_j = 1 \times [\vu_k]_j$, and $[\ve_k]_j = 1 \times [\ve_k]_j$, where $i \neq j$.

    (4) $\mathbb{E}\left[\frac{[\vu_k]_i^2}{\left \langle \vu_k, \vu_k\right \rangle}\right] = \frac{1}{n}$.

    (5) $\mathbb{E}\left[\frac{[\vu_k]_i [\vu_k]_j}{\left \langle \vu_k, \vu_k\right \rangle}\right] = 0$, where $i \neq j$.
\end{restatable}

\begin{proof}

According to real analysis, the matrix $\mA$ is full rank almost everywhere under a Gaussian distribution, and both $\vu_k$ and $\ve_k$ are non-zero almost everywhere.

(1) Since $\va_k$ is independently and identically distributed, it obviously holds.

(2) For base case $k=1$, it obviously holds. Assume the result holds for all $k= 1,2,\cdots, k-1$, where $k \ge 2$, then
$[\va_k]_i\leftrightarrow [\va_k]_j \Rightarrow [\vu_k]_i = [\va_k]_i - \sum_{s=1}^{k-1}\left \langle  \va_k, \ve_s\right \rangle [\ve_s]_i $, $[\vu_k]_j = [\va_k]_j - \sum_{s=1}^{k-1}\left \langle  \va_k, \ve_s\right \rangle [\ve_s]_j$, $[\ve_k]_i = \frac{[\vu_k]_i}{\left \| \vu_k\right \|}$, and  $[\ve_k]_j = \frac{[\vu_k]_j}{\left \| \vu_k\right \|}.$

Thus, by strong induction, we have $[\vu_k]_i\leftrightarrow [\vu_k]_j$ and $[\ve_k]_i\leftrightarrow [\ve_k]_j$.

(3) For base case $k=1$, it obviously holds. Assume the result holds for all $k= 1,2,\cdots, k-1$, where $k \ge 2$, then

\begin{align*}
[\va_k]_i = -1 \times [\va_k]_i &\Rightarrow \left\{\begin{matrix}
	[\vu_k]_i = [\va_k]_i \times (-1) - \sum_{s=1}^{k-1}\left \langle  \va_k, \ve_s\right \rangle [\ve_s]_i \times (-1) = [\vu_k]_i \times (-1) \\
	[\vu_k]_j = [\vu_k]_j \times 1, i \neq j 
\end{matrix}\right. \\
&\Rightarrow \left\{\begin{matrix}
	[\ve_k]_i \times (-1)= \frac{[\vu_k]_i}{\left \| \vu_k\right \|} \times (-1)\\
	[\ve_k]_j = [\ve_k]_j \times 1, j \neq i
\end{matrix}\right.
\end{align*}

By strong induction, we have $[\vu_k]_i = -1 \times [\vu_k]_i$ , $[\ve_k]_i = -1 \times [\ve_k]_i$, $[\vu_k]_j = 1 \times [\vu_k]_j$, and $[\ve_k]_j = 1 \times [\ve_k]_j$, where $i \neq j$.
\end{proof}

(4) Since $ \left | \frac{[\vu_k]_i^2}{\left \langle \vu_k, \vu_k\right \rangle} \right | \le 1$, $\mathbb{E}\left[\frac{[\vu_k]_i^2}{\left \langle \vu_k, \vu_k\right \rangle} \right] $ exists. $[\va_k]_i\leftrightarrow [\va_k]_j \Rightarrow \frac{[\vu_k]_i^2}{\left \langle \vu_k, \vu_k\right \rangle} \leftrightarrow \frac{[\vu_k]_j^2}{\left \langle \vu_k, \vu_k\right \rangle}$.

Thus, $\mathbb{E}\left[\frac{[\vu_k]_i^2}{\left \langle \vu_k, \vu_k\right \rangle}\right] \times n = \sum_{s=1}^{n} \mathbb{E}\left[\frac{[\vu_k]_s^2}{\left \langle \vu_k, \vu_k\right \rangle}\right] = \mathbb{E}\left[\frac{\left \langle \vu_k, \vu_k\right \rangle}{\left \langle \vu_k, \vu_k\right \rangle}\right] = 1 \Rightarrow \mathbb{E}\left[\frac{[\vu_k]_i^2}{\left \langle \vu_k, \vu_k\right \rangle}\right] = \frac{1}{n}$.

(5) Since $\left | \frac{[\vu_k]_i [\vu_k]_j}{\left \langle \vu_k, \vu_k\right \rangle} \right | \le \left | \frac{[\vu_k]_i^2 +  [\vu_k]_j^2}{2\left \langle \vu_k, \vu_k\right \rangle} \right | \le 1$, $ \mathbb{E}\left[ \frac{[\vu_k]_i [\vu_k]_j}{\left \langle \vu_k, \vu_k\right \rangle}\right]$ exists. 

$[\va_k]_i = [\va_k]_i \times -1 \Rightarrow \mathbb{E}\left[\frac{[\vu_k]_i [\vu_k]_j}{\left \langle \vu_k, \vu_k\right \rangle}\right] = \mathbb{E}\left[\frac{-[\vu_k]_i [\vu_k]_j}{\left \langle \vu_k, \vu_k\right \rangle}\right] = 0$, where $i \neq j$.

\begin{restatable}{lemma}{lemmarandomproj}
	\label{lem:randomproj}
	Let \(\mA \in \mathbb{R}^{n \times r}\) be a matrix with independent standard normal entries, i.e., each element of \(\mA\) is an i.i.d. \(\mathcal{N}(0,1)\) random variable. Suppose \(\mA\) undergoes QR decomposition via the Gram-Schmidt process to yield a column-orthogonal matrix \(\mQ \in \mathbb{R}^{n \times r}\) with orthonormal columns \(\ve_1, \ve_2, \ldots, \ve_r\) and an upper triangular matrix \(\mR \in \mathbb{R}^{r \times r}\). Then, for each \(k = 1, 2, \ldots, r\), the expected value of the outer product of the \(k\)-th orthonormal column vector \(\ve_k\) of \(\mQ\) is given by:
	\[
	\mathbb{E}[\ve_k \ve_k^T] = \frac{1}{n} \mI,
	\]
	where \(\mI\) is the \(n \times n\) identity matrix.
\end{restatable}

\begin{proof}

By the Gram-Schmidt process, we have \(\ve_k = \frac{\vu_k}{\|\vu_k\|}\), where \(\vu_k = \va_k - \sum_{s=1}^{k-1}\langle \va_k, \ve_s \rangle \ve_s\). Thus, \(\ve_k \ve_k^T = \frac{\vu_k \vu_k^T}{\langle \vu_k, \vu_k \rangle}\).

The \((i,j)\)-th entry of \(\mathbb{E}[\ve_k \ve_k^T]\) can be written as:
\[
\mathbb{E}[[\ve_k \ve_k^T]_{ij}] = \mathbb{E}\left[\frac{[\vu_k]_i [\vu_k]_j}{\langle \vu_k, \vu_k \rangle}\right].
\]

For diagonal entries (\(i = j\)):
When \(i = j\), from Lemma~\ref{lem:randomproj0}(4), we have:
\[
\mathbb{E}[[\ve_k \ve_k^T]_{ii}] = \mathbb{E}\left[\frac{[\vu_k]_i^2}{\langle \vu_k, \vu_k \rangle}\right] = \frac{1}{n}.
\]

For off-diagonal entries (\(i \neq j\)):
When \(i \neq j\), from Lemma~\ref{lem:randomproj0}(5), we have:
\[
\mathbb{E}[[\ve_k \ve_k^T]_{ij}] = \mathbb{E}\left[\frac{[\vu_k]_i [\vu_k]_j}{\langle \vu_k, \vu_k \rangle}\right] = 0.
\]

Combining these two cases, we conclude that \(\mathbb{E}[\ve_k \ve_k^T]\) is a diagonal matrix with all diagonal entries equal to \(\frac{1}{n}\). Thus,
\[
\mathbb{E}[\ve_k \ve_k^T] = \frac{1}{n} \mI,
\]
where \(\mI\) is the \(n \times n\) identity matrix.  The proof is completed. 	
\end{proof}

\begin{restatable}{lemma}{lemmarandomproj2}
	\label{lem:randomproj2}
	Let \(\mA \in \mathbb{R}^{n \times r}\) be a matrix with independent standard normal entries, i.e., each element of \(\mA\) is an i.i.d. \(\mathcal{N}(0,1)\) random variable. Suppose \(\mA\) undergoes QR decomposition to yield an orthogonal matrix \(\mQ \in \mathbb{R}^{n \times r}\) with orthonormal columns and an upper triangular matrix \(\mR \in \mathbb{R}^{r \times r}\). Then, the expected value of the outer product of the matrix \(\mQ\) with itself is given by:
	\[
	\mathbb{E}[\mQ \mQ^T] = \frac{r}{n} \mI
	\]
	where \(\mI\) is the \(n \times n\) identity matrix.
\end{restatable}

\begin{proof}
	
	The QR decomposition of \(\mA\) is given by \(\mA = \mQ \mR\), where \(\mQ\) is an orthogonal matrix with columns \(\ve_1, \ve_2, \ldots, \ve_r\) and \(\mR\) is an upper triangular matrix. Since \(\mQ\) is orthogonal, \(\mQ \mQ^T = \mI_r\), where \(\mI_r\) is the \(r \times r\) identity matrix. We aim to compute \(\mathbb{E}[\mQ \mQ^T]\). By linearity of expectation and the fact that the columns of \(\mQ\) are orthonormal, we have:
	\[
	\mathbb{E}[\mQ \mQ^T] = \mathbb{E}\left[\sum_{k=1}^r \ve_k \ve_k^T\right] = \sum_{k=1}^r \mathbb{E}[\ve_k \ve_k^T].
	\]
	From Lemma~\ref{lem:randomproj}, we know that \(\mathbb{E}[\ve_k \ve_k^T] = \frac{1}{n} \mI\) for each \(k\). Therefore:
	\[
	\mathbb{E}[\mQ \mQ^T] = \sum_{k=1}^r \frac{1}{n} \mI = \frac{r}{n} \mI.
	\]
	The proof is completed. 	
\end{proof}

\begin{restatable}{lemma}{lemmarandomproj3}
	\label{lem:randomproj3}
	Let \(\mA_1 \in \mathbb{R}^{m \times r}\) and \(\mA_2 \in \mathbb{R}^{n \times r}\) be matrices with independent standard normal entries, i.e., each element of \(\mA_1\) and \(\mA_2\) is an i.i.d. \(\mathcal{N}(0,1)\) random variable. Suppose \(\mA_1\) and \(\mA_2\) undergo QR decomposition to yield orthogonal matrices \(\mQ_1 \in \mathbb{R}^{m \times r}\) and \(\mQ_2 \in \mathbb{R}^{n \times r}\) with orthonormal columns, respectively. Define \(\mP = \mQ_2 \otimes \mQ_1\), where \(\otimes\) denotes the Kronecker product. Then, the expected value of the outer product of the matrix \(\mP\) with itself is given by:
	\[
	\mathbb{E}[\mP \mP^T] = \frac{r^2}{mn} \mI,
	\]
	where \(\mI\) is the \(mn \times mn\) identity matrix.
\end{restatable}

\begin{proof}
	
	The Kronecker product \(\mP = \mQ_2 \otimes \mQ_1\) results in a matrix \(\mP \in \mathbb{R}^{mn \times r^2}\). 
	From Lemma~\ref{lem:randomproj2}, we have \(\mQ_1 \mQ_1^T = \frac{r}{m}\mI\) and \(\mQ_2 \mQ_2^T = \frac{r}{n}\mI\). We aim to compute \(\mathbb{E}[\mP \mP^T]\). Using the properties of the Kronecker product, we have:
	\[
	\mathbb{E}[\mP \mP^T] = \mathbb{E}[(\mQ_2 \otimes \mQ_1)(\mQ_2^T \otimes \mQ_1^T)] = \mathbb{E}[(\mQ_2 \mQ_2^T)] \otimes \mathbb{E}[(\mQ_1 \mQ_1^T)] = \frac{r^2}{mn}\mI \otimes \mI = \frac{r^2}{mn}\mI
	\]
	The proof is completed. 	
\end{proof}

Now we can  assess the impact of the lazy updates to \(\mP\).

\theoremconvergence*
\begin{proof}
	
	Suppose $\mathcal{P}_j = (\mP_0, \mP_1, \cdots, \mP_j)$, where $\mP_j$ is the sequence generated by Eqn.~\eqref{eq:def_P} and $j \leq K$. In accordance with Lemma~\ref{lem:preserveconvexandsmooth} and Lemma~\ref{lem:sub_convergence}, if the subspace is fixed, we can transform the original problem $f\in C^{1,1}_{L_1}(\mathbb{R}^d)$ into $h \in C^{1,1}_{L_1}(\mathbb{R}^q)$ through transformation $h(\vy) = f(\vx + \mP\vy)$.  Consider the update rule:
    \begin{align}
        \vy_{j, 0} &= 0, h_j(\vy) = f(\vx_{jF} + \mP_j \vy), \forall j \in {0,1,\cdots, K-1} \\
        \vy_{j, k}  &= \vy_{j, k-1} - \eta \widehat{\nabla}h_j(\vy_{j,k-1}), \forall  k \in {0,1,\cdots,F} \\
        \vx_{jF+k} &= \vx_{jF} + \mP_j \vy_{k},
    \end{align}
	 In the $j$-th subspace, the projection matrix $\mP_j$ remains constant, hence we can accumulate the changes of $\phi$ within the current subspace. Using Lemma~\ref{lem:sub_convergence}, we have
	\begin{align}
		\phi_{(j+1)F} - \phi_{jF} &\leq -\frac{1}{4} \hat{\eta} \sum_{i=0}^{K-1} \mathbb{E}_{\mathcal{E}_{jF+i}} \left[\|\nabla h_j(\vy_{j, i})\|^{2}\right] + \frac{\varepsilon^{2} (q + 6)^{3}}{8} K L_{1}^{2} + \frac{3 \varepsilon^{2} (q + 4)}{32} K L_{1}\\
         &\leq -\frac{1}{4} \hat{\eta} \mathbb{E}_{\mathcal{E}_{jF}} \left[\|\nabla h_j(\vy_{j, 0})\|^{2}\right] + \frac{\varepsilon^{2} (q + 6)^{3}}{8} K L_{1}^{2} + \frac{3 \varepsilon^{2} (q + 4)}{32} K L_{1}.
	\end{align}
	Additionally, we note that \(\nabla h_j(\vy_{j, 0}) = (\mP_{j})^\mathsf{T} \nabla f(\vx_{jF})\).  Taking expectations over the overall historical projection matrix \(\mathcal{P}_j\), and noting Lemma~\ref{lem:randomproj3}, \(\mathbb{E}[\mP_{j} (\mP_{j})^\mathsf{T}] = \frac{q}{d} \mI\), with \(\mP_j\) independent of \(\vx_{jF}\), we get
	\begin{align}
		\mathbb{E}_{\mathcal{P}_{j+1}}[\phi_{(j+1)F}] - \mathbb{E}_{\mathcal{P}_j}[\phi_{jF}] &\leq - \frac{1}{4} \hat{\eta} \mathbb{E}_{\mathcal{E}_{jF}, \mathcal{P}_j} \left[\|(\mP_{j})^\mathsf{T} \nabla f(\vx_{jF})\|^{2}\right] + \frac{\varepsilon^{2} (q + 6)^{3}}{8} KL_{1}^{2} + \frac{3 \varepsilon^{2} (q + 4)}{32} KL_{1} \\
		&= -\frac{q}{4d} \hat{\eta} \mathbb{E}_{\mathcal{E}_{jF}, \mathcal{P}_j} \left[\| \nabla f(\vx_{jF})\|^{2}\right] + \frac{\varepsilon^{2} (q + 6)^{3}}{8} KL_{1}^{2} + \frac{3 \varepsilon^{2} (q + 4)}{32} KL_{1}. 
	\end{align}
     Assuming \(f(\vx) \geq f^{*}\) holds for all \(\vx \in \mathbb{R}^d\), and letting \(T = KF\), summing the inequality yields
	\begin{align}
		\mathbb{E}_{\mathcal{P}_{K-1}}[\phi_{T}] \leq \mathbb{E}_{\mathcal{P}_{0}}[\phi_{0}] - \frac{q}{4d} \hat{\eta} \sum_{j=0}^{K-1}  \mathbb{E}_{\mathcal{E}_{jF}, \mathcal{P}_j} \left[\| \nabla f(\vx_{jF}) \|^{2}\right] + T \frac{\varepsilon^{2} (q + 6)^{3}}{8} L_{1}^{2} + T \frac{3 \varepsilon^{2} (q + 4)}{32} L_{1}.
	\end{align}
	Since \(\mathbb{E}_{\mathcal{P}_{K-1}}[\phi_{T}]  \geq f^*\), we have:
	\begin{align}
		f^* &\leq \mathbb{E}_{\mathcal{P}_0}[\phi_{0}] - \frac{q}{4d} \hat{\eta} \sum_{j=0}^{K-1}  \mathbb{E}_{\mathcal{E}_{jF}, \mathcal{P}_j} \left[\| \nabla f(\vx_{jF}) \|^{2}\right] + T \frac{\varepsilon^{2} (q + 6)^{3}}{8} L_{1}^{2} + T \frac{3 \varepsilon^{2} (q + 4)}{32} L_{1}.
	\end{align}
	Rearranging the inequality, we get
	\begin{align}
		\frac{q}{4d}\hat{\eta} \sum_{j=0}^{K-1}  \mathbb{E}_{\mathcal{E}_{jF}, \mathcal{P}_j} \left[\| \nabla f(\vx_{jF}) \|^{2}\right] &\leq \mathbb{E}_{\mathcal{P}_0}[\phi_{0}] - f^{*} + T \frac{\varepsilon^{2} (q + 6)^{3}}{8} L_{1}^{2} + T \frac{3 \varepsilon^{2} (q + 4)}{32} L_{1}.
	\end{align}
	Substituting \(\hat{\eta} = \frac{1}{4(q+4)L_{1}}\), we obtain:
	\begin{align}
		\frac{q}{16d(q+4)L_{1}} \sum_{j=0}^{K-1} \mathbb{E}_{\mathcal{E}_{jF}, \mathcal{P}_j} \left[\| \nabla f(\vx_{jF}) \|^{2}\right] &\leq \mathbb{E}_{\mathcal{P}_0}[\phi_{0}] - f^{*} + T \frac{\varepsilon^{2} (q + 6)^{3}}{8} L_{1}^{2} + T \frac{3 \varepsilon^{2} (q + 4)}{32} L_{1}.
	\end{align}
	Thus, we have
	\begin{align}
		\frac{1}{T} \sum_{k=0}^{T - 1} \mathbb{E}_{\mathcal{E}_{k}, \mathcal{P}_{\left \lfloor k / F \right \rfloor }} \left[\| \nabla f(\vx_{k}) \|^{2}\right] \leq \frac{16 (q + 4) d L_{1} (\mathbb{E}_{\mathcal{P}_0}[\phi_{0}] - f^{*})}{q T }  + \frac{2 \varepsilon^{2} (q + 6)^{3} (q + 4) d}{q} L_{1}^{3} + \frac{3 \varepsilon^{2} (q + 4)^{2} d}{2q} L_{1}^{2}.
	\end{align}
	To ensure \(\sum_{k=0}^{T - 1} \mathbb{E}_{\mathcal{E}_{k}, \mathcal{P}_{\left \lfloor k / F \right \rfloor }} \left[\| \nabla f(\vx_{k}) \|^{2}\right] \leq \epsilon\), we can choose
	\[
	\varepsilon \leq \mathcal{O} \left(\frac{\epsilon^{1/2}}{q^{3/2} d^{1/2} L_{1}^{3/2}}\right).
	\]
	Then, the upper bound for the expected number of steps is $\mathcal{O}\left(\frac{d}{\epsilon}\right)$.
	The proof is completed.
\end{proof}

\end{document}